\documentclass[lettersize,journal]{IEEEtran}
\usepackage[T1]{fontenc} 
\usepackage{amssymb, amsmath, amsthm, bm}
\usepackage{graphicx}
\usepackage[caption=false]{subfig}
\usepackage{epstopdf}
\usepackage[colorlinks,linkcolor=blue]{hyperref}
\usepackage{bbding}
\usepackage{multirow}
\usepackage{booktabs}
\usepackage{tikz,xcolor}
\usepackage{wrapfig}
\usepackage{cite}
\usepackage{enumitem}

\definecolor{lime}{HTML}{A6CE39}
\DeclareRobustCommand{\orcidicon}{%
	\begin{tikzpicture}
	\draw[lime, fill=lime] (0,0)
	circle [radius=0.16]
	node[white] {{\fontfamily{qag}\selectfont \tiny ID}};    \draw[white, fill=white] (-0.0625,0.095)
	circle [radius=0.007];    \end{tikzpicture}
	\hspace{-2mm}}
\foreach \x in {A, ..., Z}{%
	\expandafter\xdef\csname orcid\x\endcsname{\noexpand\href{https://orcid.org/\csname orcidauthor\x\endcsname}{\noexpand\orcidicon}}}

\newcommand{\etal}{\textit{et~al.}}
\newcommand{\ie}{\textit{i.e.}}
\newcommand{\eg}{\textit{e.g.}}

\newtheorem{prop}{Proposition}

\newtheorem{rmrk}{Remark}

\numberwithin{equation}{section}

\everymath{\displaystyle}
\begin{document}

\title{Understanding Adversarial Robustness from Feature Maps of Convolutional Layers}

\author{
        Cong~Xu\orcidA,
		Wei~Zhang\orcidB,
		Jun~Wang\orcidC,
        Min~Yang\orcidD
\thanks{
	This research is partially supported by National Natural Science Foundation of China (11771257) and Natural Science Foundation of Shandong Province (ZR2021MA010).
	(\emph{Corresponding author: Min Yang.})
}
\thanks{
	Cong Xu is with the School of Mathematics and Information Sciences,
	Yantai University, Yantai 264005, China, and with the School of Computer Science and Technology,
	East China Normal University, Shanghai, China. (email: congxueric@gmail.com)
}
\thanks{
	Wei Zhang and Jun Wang are with the School of Computer Science and Technology,
	East China Normal University, Shanghai, China. (email: zhangwei.thu2011@gmail.com, wongjun@gmail.com)
}
\thanks{
	Min Yang is with the School of Mathematics and Information Sciences,
	Yantai University, Yantai 264005, China. (email: yang@ytu.edu.cn)
}
}

\maketitle
		
\begin{abstract}
    The adversarial robustness of a neural network mainly relies on two factors: model capacity and anti-perturbation ability.
    In this paper, we study the anti-perturbation ability of the network from the feature maps of convolutional layers.
    Our theoretical analysis discovers that larger convolutional feature maps before average pooling can contribute to better resistance to perturbations,
    but the conclusion is not true for max pooling.
    It brings new inspiration to the design of robust neural networks and urges us to apply these findings to improve existing architectures.
    The proposed modifications are very simple and only require upsampling the inputs or slightly modifying the stride configurations of downsampling operators.
    We verify our approaches on several benchmark neural network architectures, including AlexNet, VGG, RestNet18, and PreActResNet18.
    Non-trivial improvements in terms of both natural accuracy and adversarial robustness can be achieved under various attack and defense mechanisms.
	The code is available at \url{https://github.com/MTandHJ/rcm}.
\end{abstract}		


\section{Introduction}

\IEEEPARstart{A}{lthough} deep neural networks (DNNs) have achieved compelling performance on many challenging learning tasks \cite{wu2023aggn,wu2023kdpar},
some work \cite{goodfellow2015,szegedy2013} found that they are vulnerable to artificially crafted adversarial perturbations.
Imposing human-imperceptible perturbations on clean samples could deceive networks and cause incorrect classification.
As a result, the vulnerability of DNNs impedes their deployment, especially in security-critical applications.

In order to evaluate the robustness of networks,
a series of attack methods \cite{goodfellow2015,moosavi2016deepfool,madry2018pgd,carlini2017cwl2,croce2020aa} have been developed.
There is a consensus among them that,
standardly trained DNNs have little robustness against human-imperceptible perturbations.
To alleviate this problem, numerous defense methods \cite{naseer2020ssp,wong2020fgsmrs,xu2022tnnls,yang2021ddrrn,zhai2022tnnls}
have been proposed to improve the robustness of DNNs.
In particular, adversarial training~\cite{madry2018pgd,zhang2019trades} enjoys superior performance
by enforcing DNNs to correctly classify not only clean samples (for natural accuracy),
but also their artificially perturbed counterparts (for adversarial robustness).

\begin{figure}
	\centering
	\centering
	\scalebox{.85}{\includegraphics{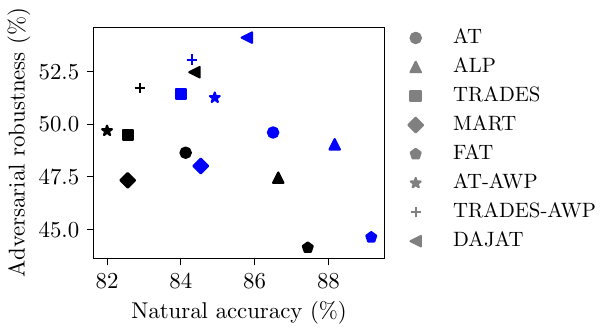}}
	\caption{
		 Natural accuracy  (\%) and adversarial robustness (\%) on CIFAR-10.
		 Different defense mechanisms are applied to the \textbf{baseline}  and \textcolor{blue}{improved} ResNet18.
	}
	\label{fig-basicComp}
\end{figure}

We argue that the robustness of a neural network mainly relies on two factors:
one is the model capacity to achieve high natural accuracy, and the other is its anti-perturbation ability to make consistent predictions about perturbed counterparts.
Usually larger neural networks tend to have better model capacity~\cite{cao2019},
which is why most robust results are based on large network architectures  \cite{gowal2020wa,madry2018pgd,wu2021wider}.
However, a larger network does not necessarily mean a better anti-perturbation ability.
On the contrary,
a larger network is likely to have a larger Lipschitz constant (in which case it is sensitive to the imperceptible perturbations),
which in turn leads to lower resistance to perturbations \cite{huang2021wideresnetR,wu2021wider}.
Noticeably, a peculiar phenomenon often observed in adversarial robustness literature is that when the natural accuracy of a model increases,
the corresponding robustness decreases,
which also indicates that the anti-perturbation ability of the network is not consistent with its model capacity.

As an essential component of most Convolutional Neural Netwoks (CNNs), the convolutional part has a key impact on the final performance.
Typically, this part maps an image into convolutional feature maps and then a pooling operator (\eg, average or max pooling) is used to connect them to the subsequent classifier.
This paper aims to analyze the anti-perturbation ability of CNNs from the perspective of convolutional feature maps.
Our theoretical analysis reveals that average pooling in conjunction with larger convolutional feature maps helps to increase anti-perturbation ability,
but the conclusion is not true for max pooling.

Based on these theoretical understandings, we present two simple but effective ways to improve existing CNNs.
The first straightforward approach is to upsample the input such that subsequent feature maps are scaled up.
The other is to modify sliding stride configurations of downsampling operators.
Both are orthogonal to existing defense mechanisms,
from the simplest AT \cite{madry2018pgd} and TRADES \cite{zhang2019trades} to the state-of-the-art AWP~\cite{wu2020awp} and DAJAT \cite{addepalli2020dajat}.
With the proposed modifications,
we achieve impressive robustness improvements for benchmark CNNs such as  AlexNet \cite{krizhevsky2012alex}, VGG16 \cite{simonyan2015vgg}, RestNet18 \cite{he2016resnet}, and PreActResNet18 \cite{he2016preactresnet}.
Besides, our further attempts at the Transformer architecture \cite{dosovitskiy2021vits} show that the approach is fairly effective and generalizable.
Moreover, in addition to adversarial robustness, we also observe a surprising enhancement of natural accuracy on the modified networks.
It indicates that the proposed methods not only help to strengthen the anti-perturbation ability of neural networks, but can also boost the model capacity.
We believe that the study of the paper will lead to a new perspective to understand the adversarial robustness of neural networks.

Our major contributions can be summarized as follows:
\begin{itemize}
	\item
	We theoretically reveal that the anti-perturbation ability of neural networks is closely related to the feature maps of convolutional layers.
	Properly enlarging the dimensions of convolutional features before average pooling can improve robustness,
    whereas the opposite holds true for max pooling.

	\item
    Two effective ways with slight modifications on convolutional layers are presented to improve existing CNNs.
	We carefully compare the resulting model capacity, anti-perturbation ability, and computational efficiency.
    Significant improvements are observed in terms of both robustness and natural accuracy.

	\item
	Previous studies usually improve adversarial robustness at the cost of natural accuracy.
	It is because the architectures they rely on have less consideration concerning the anti-perturbation ability.
    The research expands the understanding of anti-perturbation ability and brings in new insights for developing robust neural networks.

\end{itemize}

\section{Related Works}

\begin{table}[htb]
	\caption{
			Overview of related works.
	}
	\label{table-related-works}
	\centering
	\scalebox{1}{
	\begin{tabular}{cc}
			\toprule
			& Methods \\
			\midrule
			Adversarial training & \cite{madry2018pgd,zhang2019trades,wang2020mart,zhang2020fat,gowal2020wa,pang2020bag,zhai2022tnnls,xu2022tnnls,kanai2023relation} \\
			Certified robustness & \cite{cohen2019defense,salman2019smooth,liu2021provably} \\
			Input manipulation & \cite{rebuffi2021fix,lee2023booster,wang2023diffusion} \\
			Architectural design  & \cite{wu2021wider,huang2021wideresnetR,dai2021pssilu,shao2021robustofvit,zecchin2023emsemble,peng2023rwide} \\
			\bottomrule
	\end{tabular}}
\end{table}

In this section, we briefly review the related works on adversarial robustness.
Table \ref{table-related-works} presents an overview of defense methods.

\subsection{Robust learning strategies}

Numerous robust learning algorithms have been developed from different perspectives to help neural networks achieve robust feature representations.
For example, Cohen \etal~\cite{cohen2019defense} and Salman \etal~\cite{salman2019smooth} have studied the randomized smoothing technique
that can transform any base classifier into a new smoothed classifier with certifiable robustness.
Pang~\etal~\cite{pang2020mmc} exploited a Max-Mahalanobis center loss to force the model to learn compact features.
Adversarial training \cite{madry2018pgd} is the most successful robust learning technique,
which aims to dynamically fit the underlying distribution of adversarial samples by a min-max optimization objective.
Specifically, TRADES~\cite{zhang2019trades} and MART~\cite{wang2020mart} contribute to a good trade-off between natural accuracy and adversarial robustness.
FAT \cite{zhang2020fat} minimizes a loss over friendly adversarial samples, thereby enjoying superior natural accuracy.
AIB \cite{zhai2022tnnls} and InfoAT \cite{xu2022tnnls} improve adversarial training through the lens of information bottleneck.
Nevertheless, adversarial training is known to suffer from over-fitting \cite{rice2020overfitting}.
Therefore, an early-stopping learning schedule \cite{gowal2020wa,pang2020bag} is needed.
In addition, some practical techniques such as adversarial weight perturbation (AWP) \cite{wu2020awp}  and weight averaging \cite{gowal2020wa}
can help circumvent the local optima.

With the development of robust feature learning methods,
the progress that can be made becomes more and more limited.
Some works turn to input manipulation \cite{lee2023booster},
especially data augmentation \cite{rebuffi2021fix} and data generalization \cite{wang2023diffusion} to account for some rare cases.
However, these efforts still remain lacking in the understanding of the architecture itself.

\subsection{Robust network architectures}

The robustness of a neural network mainly relies on two factors: model capacity and anti-perturbation ability.
Naturally, one of promising directions is to promote anti-perturbation ability by using certain large networks.
Unfortunately, deeper or wider networks are likely to have larger Lipschitz constants,
which leads to lower resistance to perturbations \cite{wu2021wider,huang2021wideresnetR}.
As a result, for the widely-used WideResNet~\cite{zagoruyko2016wideresnet},
reducing its depth and width at the last stage in turn increases adversarial robustness~\cite{huang2021wideresnetR}.

As such, anti-perturbation ability is of particular value for adversarial robustness.
Recently, Dai \etal \cite{dai2021pssilu} have studied the learnable parametric activation functions
and the developed PSSiLU can significantly increase robustness when extra training samples are available.
Shao \etal \cite{shao2021robustofvit} have showed that vanilla ViTs \cite{dosovitskiy2021vits} can learn more generalizable features
and thus has superior robustness against adversarial perturbations.
Peng \etal \cite{peng2023rwide} developed an `optimal' architecture by exhaustive searching over different modules,
which achieves the state-of-the-art performance.

Although the aforementioned studies have made impressive robustness improvements in exploring network architectures,
the research on designing robust network architectures is far from enough.
These modifications still lack theoretical guarantees of effectiveness and are therefore difficult to apply across different scenarios.
Different from the perspective of existing studies,
this paper investigates the anti-perturbation ability of neural networks through convolutional feature maps.
The proposed simple modifications can be used to improve most of well-established architectures.
We anticipate that this work will encourage further efforts to understand and develop robust architectures.

\section{Anti-Perturbation Ability of Convolutional Feature Maps}

In this section, we study the anti-perturbation ability of the convolutional feature maps after pooling.
Before delving into the theoretical analysis, we first reveal the vital role of anti-perturbation ability for robustness.

\subsection{Preliminaries}

A neural network $ f(\cdot) $ is robust around a clean sample $x$ with the label $y $,
if the following two conditions are satisfied.
First, it recognizes the sample correctly, \ie,
\begin{align}
	\label{eq-correct}
	 f(x) = y.
\end{align}
Second, it predicts its perturbed counterparts consistently within a given perturbation budget $\epsilon$, \ie,
\begin{align}
	\label{eq-stable}
	f(x + \delta) = f(x), \quad \forall \: \|\delta\|_{\infty} \le \epsilon.
\end{align}

For the entire dataset,
the proportion of correctly identified clean samples is exactly the natural accuracy,
which reflects the model capacity of the neural network,
while the proportion of consistently classified perturbed samples empirically reflects the anti-perturbation ability.

Normally, the anti-perturbation ability can be estimated by the following formula:
\begin{align}
	\|f(x) - f(x')\|_{\infty} \le L \|x - x' \|_{\infty},
\label{Lip}
\end{align}
where $ L $ is a Lipschitz constant.
A small value indicates the high anti-perturbation ability of the network.

\begin{rmrk}
The adversarial robustness of a network is actually determined by its model capacity Eq. \eqref{eq-correct} and anti-perturbation ability Eq. \eqref{eq-stable}.
Larger neural networks usually possess better model capacity,
which consequently raises the upper limit of robustness.
However, as pointed out by recent studies \cite{huang2021wideresnetR,wu2021wider},
large neural networks may suffer from a degradation in perturbation resistance due to potentially large Lipschitz constants.
\end{rmrk}

In the following section,
we mainly focus on the anti-perturbation ability of the network from the perspective of the convolutional feature maps.

\subsection{Connection between anti-perturbation ability and convolutional feature maps}
\label{section-theory}

A typical deep convolutional network $ f(\cdot) $ for classification can be depicted as
\begin{align*}
   f(x) = h \circ \mathcal{P} \circ g(x).
\end{align*}
The convolutional encoder $g$ maps the image $x$ into the feature maps $g(x) \in \mathbb{R}^{C \times H \times W}$,
in which $C$ is the number of channels and the sizes $H, W$ respectively denote the height and width.
Then, the classifier $h$ predicts the category of the image based on the pooled features,
where the pooling operator $\mathcal{P}$ bridges the two in a \textit{channel-wise} manner.
To achieve high classification accuracy, the pooled features should be discriminative,
which is the main concern of the convolutional encoder.
In other words, the anti-perturbation ability of $\mathcal{P} \circ g$ largely determines the robustness of the entire model.
Next, we will study the property of the encoder in detail.

Let $ \delta $ denote the \textit{random} noise added to the image $x$ and
\begin{align*}
\Delta = g(x+\delta)-g(x) \in \mathbb{R}^{C \times H \times W}
\end{align*}
be the difference between the perturbed and the clean ones.
The pooled features can be given by
\begin{align*}
\mathcal{P}(\Delta) = \mathcal{P} \circ g(x+\delta)- \mathcal{P} \circ g(x) \in \mathbb{R}^C.
\end{align*}
For the convenience of discussion, we assume that the number of channels $ C=1 $; however,
the conclusions below can be extended to the general case since both average and max pooling here are channel-wise operations.

Let $ \gamma $ denote the threshold of interest.
The following inequality measures the anti-perturbation ability of  $ \mathcal{P} \circ g $ from a probabilistic point of view:
\begin{align}
		\mathbb{P}(|\mathcal{P}(\Delta)| \ge \gamma) \le p,
\label{staIneq}
\end{align}
where  $ 0\leq p \leq 1 $ is a probability that signifies the risk of the magnitude $|\mathcal{P}(\Delta)|$ exceeding the threshold $\gamma$.
For a given $p$ and a perturbation $\delta $ from a certain distribution,
the smaller $ \gamma $ is,
the better the anti-perturbation stability is.
The proposition below formally gives two sufficient conditions for Eq. \eqref{staIneq} to hold,
one based on average pooling and the other based on max pooling.

\begin{prop}
	\label{proposition-stability}
	For any sample $ x $, let $ g(x) \in \mathbb{R}^{H \times W}$ denote the corresponding convolutional feature map.
    Assume that the random perturbation $\delta$ follows a probability distribution $\mathcal{D}$ such that
    the elements of $ \Delta(\delta) = g(x+\delta)-g(x) $ are independent from each other
    and $\mathbb{E}_{\delta} [\Delta_{ij}(\delta)] = 0$ for all $ 1 \le i \le H, 1 \le j \le W$.
	Let
	 \begin{align*}
		 a = \min_{
			\substack{
				i, j \\
			  \delta \sim \mathcal{D} \\
			}
		 } \: \Delta_{ij}(\delta), \quad
		 b = \max_{
			\substack{
				i, j \\
			\delta \sim \mathcal{D} \\
			}
		 } \: \Delta_{ij}(\delta).
	 \end{align*}
	Then Eq. \eqref{staIneq} holds true when $\mathcal{P} $ is average pooling and
	\begin{align}
		\label{eq-average-pooling}
		2\exp \bigg( -\frac{2HW \gamma^2}{(b-a)^2} \bigg) \le p,
	\end{align}
    or when $\mathcal{P} $ is max pooling  and
	\begin{align}
		\label{eq-max-pooling}
		\frac{(b - a)\sqrt{\log\sqrt{2HW}}}{\gamma} \le p.
	\end{align}
\end{prop}

\begin{proof}
    Firstly,
    notice that $ \mathcal{P}(\Delta)=\frac{1}{HW}\sum_{i,j=1}^{H,W} \Delta_{i,j} $ if the average pooling is applied.
    Applying the Hoeffding's inequality~\cite{boucheron2013ieq} on the interval $[\frac{a}{HW}, \frac{b}{HW}]$,
    we have
	\begin{align*}
		\mathbb{P}(|\mathcal{P}(\Delta)| \ge \gamma) \le 2\exp \bigg( -\frac{2HW \gamma^2}{(b-a)^2} \bigg).
	\end{align*}
    Therefore, Eq. \eqref{staIneq} holds if
	\begin{align*}
		2\exp \bigg( -\frac{2HW \gamma^2}{(b-a)^2} \bigg) \le p,
	\end{align*}
	which comes to our first conclusion Eq. \eqref{eq-average-pooling}.

	Secondly,
    note that $ \mathcal{P}(\Delta)=\max_{1\leq i \leq H, 1\leq j \leq W} \Delta_{i,j} $ if the max pooling is applied.
    Define the following logarithmic moment-generating function $\psi_{\Delta_{i,j}} (\lambda) = \log \mathbb{E}_\delta [e^{\lambda \Delta_{i,j}}]$,
    where  $\lambda $ is a variable to be determined later.
    Using the Hoeffding's lemma~\cite{boucheron2013ieq} yields
	\begin{align}
		\label{eq-lmgf}
		\psi_{\Delta_{ij}}(\lambda) \le \frac{\lambda ^2 (b - a)^2}{8}, \quad \forall i,j.
	\end{align}
	Then, it follows from the Jensen's inequality \cite{rudin1978real} and Eq.~\eqref{eq-lmgf} that
	\begin{align*}
		\begin{array}{ll}
		  &\exp{(\lambda \mathbb{E}_\delta [\max_{i,j} |\Delta_{i,j}|])} \\
		\le & \mathbb{E}_\delta [\exp (\lambda \max_{i,j} |\Delta_{i,j}|)] \\
		\le &\sum_{i,j=1}^{H,W} (\mathbb{E}_\delta [\exp(\lambda \Delta_{ij})] + \mathbb{E}_{\delta} [\exp (-\lambda \Delta_{ij})]) \\
        \le &2HW \exp (\lambda^2 (b - a)^2 / 8).
		\end{array}
	\end{align*}
	Taking logarithm on both sides yields
	\begin{align*}
		\mathbb{E}_\delta [\max_{i,j} |\Delta_{ij}|] \le \frac{\log 2HW}{\lambda} + \frac{\lambda (b - a)^2}{8}.
	\end{align*}
	Setting $\lambda = \sqrt{8 \log 2HW / (b - a )^2}$  gives
	\begin{align*}
		\mathbb{E}_{\delta} [\max_{i,j} |\Delta_{ij}|] \le  \sqrt{\frac{(b - a)^2 \log 2 HW}{2}}.
	\end{align*}
	Using the Markov's inequality \cite{boucheron2013ieq} and the above estimation, we have
	\begin{align*}
		\mathbb{P}(|\mathcal{P}(\Delta)| \ge \gamma)
		&\le \mathbb{P}(\max_{i,j} |\Delta_{ij}| \ge \gamma) \\
		&\le \frac{\mathbb{E}_\delta [\max_{i,j} |\Delta_{ij}|] }{\gamma} \\
		&\le \frac{(b - a)\sqrt{\log\sqrt{2HW}}}{\gamma}.
	\end{align*}
	Therefore, Eq. \eqref{staIneq} holds if
	\begin{align*}
		\frac{(b - a)\sqrt{\log\sqrt{2HW}}}{\gamma} \le p,
	\end{align*}
	which comes to our second conclusion Eq. \eqref{eq-max-pooling}.
\end{proof}

\begin{rmrk}
There are two assumptions regarding $ \Delta $ in Proposition \ref{proposition-stability}.
One is the zero-mean assumption, which can be satisfied in some cases, such as when the condition $\mathbb{E}[\delta] = 0$ holds and the encoder $g$ is linear.
The other is about the independence between the feature map entries.
This is a somewhat strict requirement not easy to meet in practice.
However, our experiments show that the proposed modifications are still effective even if the network does not satisfy these assumptions.
\end{rmrk}

\begin{rmrk}
Proposition \ref{proposition-stability} implies that, in the context of average pooling,
as the feature sizes $H$ and $W$ increase-indicating better model capacity-a more stringent threshold $\gamma$ is needed for Eq.~\eqref{eq-average-pooling} to hold.
This tighter threshold signifies enhanced anti-perturbation ability.
Therefore, enlarging feature sizes before applying average pooling may concurrently boost both model capacity and anti-perturbation ability,
thereby leading to improved adversarial robustness.
\\
In contrast, max pooling presents a trade-off: larger feature maps (indicative of better model capacity) require a larger $\gamma$ (indicating worse anti-perturbation ability) to meet Eq. ~\eqref{eq-max-pooling}.
Conversely, opting for a smaller size allows for tighter thresholds but often results in diminished model capacity.
Note that adversarial robustness depends on the combined influence of model capacity and anti-perturbation ability.
However, this intrinsic trade-off complicates the determination of an optimal feature size for maximum pooling.
\end{rmrk}

\begin{rmrk}
It should be noticed that increasing feature map sizes before average pooling can not always guarantee better anti-perturbation ability,
because the term $(b - a)$ is also positively related to $HW$.
Too large sizes may result in a smaller $\frac{HW}{(b - a)^2}$.
In this case, the condition Eq.~\eqref{eq-average-pooling} holds only for some meaninglessly large thresholds.
Nonetheless, as can be seen from the experiments in Section \ref{section-experiments},
the designs of most neural networks are far from optimal.
\end{rmrk}

In what follows., we will take the average pooling operator as the first option.
It is worth mentioning that the average pooling here refers specifically to the operator bridging the convolutional encoder and classifier.
For networks that adopt other operations for this goal, we can replace it with average pooling for better robustness.

The toy example illustrated in Figure~\ref{fig-perturbation-toy}  empirically verifies our findings (see Section \ref{section-empirical-analysis} for design details).
The disturbance after average pooling is small, and becomes further smaller as the feature size increases.
On the contrary, once max pooling is applied, the output disturbance gets larger with the increase of feature size.

\begin{figure}[hbt]
	\center
	\includegraphics[width=.45\textwidth]{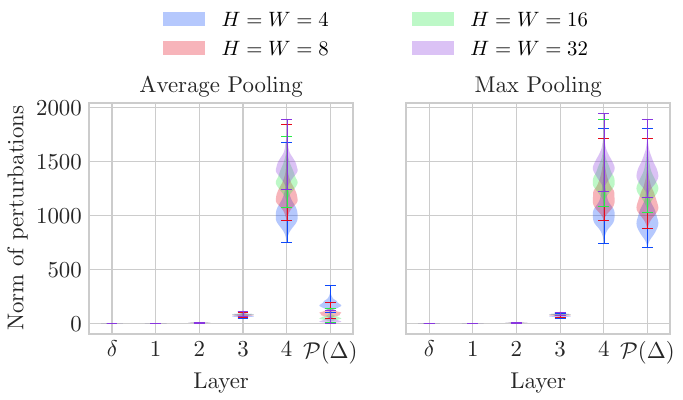}
	\caption{
		The $\ell_{\infty}$ norm of input perturbations from $\mathcal{U}[-0.1, 0.1]$ and the resulting disturbances at each layer.
		\textbf{Left:} A randomly initialized CNN followed by average pooling;
		\textbf{Right:} A randomly initialized CNN followed by max pooling.
	}
	\label{fig-perturbation-toy}
\end{figure}

\section{Feasible Ways to Enlarge Feature maps}

In this section, we present two feasible ways to enlarge convolutional feature maps: \textit{upsampling inputs} versus \textit{shrinking sliding strides}.
For most existing CNNs, they are easy-to-use and plug-and-play; however, applying them correctly and effectively needs more investigation.

\subsection{Upsampling inputs}
\label{section-upsampling}

\begin{figure}
	\centering
	\subfloat[]{\label{fig-upsampling}\includegraphics[width=0.24\textwidth]{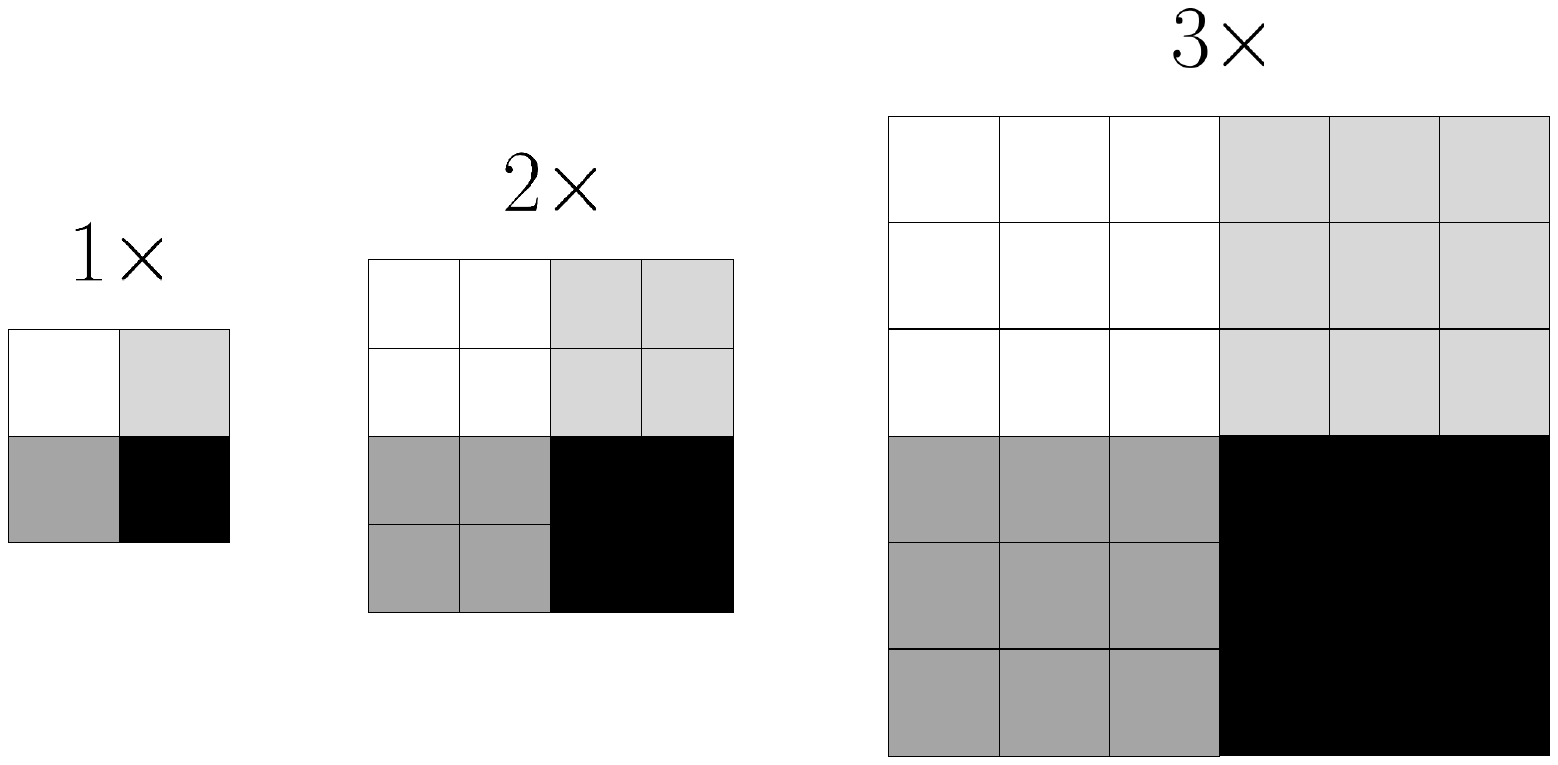}}
	\subfloat[]{\label{fig-redundancy}\includegraphics[width=0.24\textwidth]{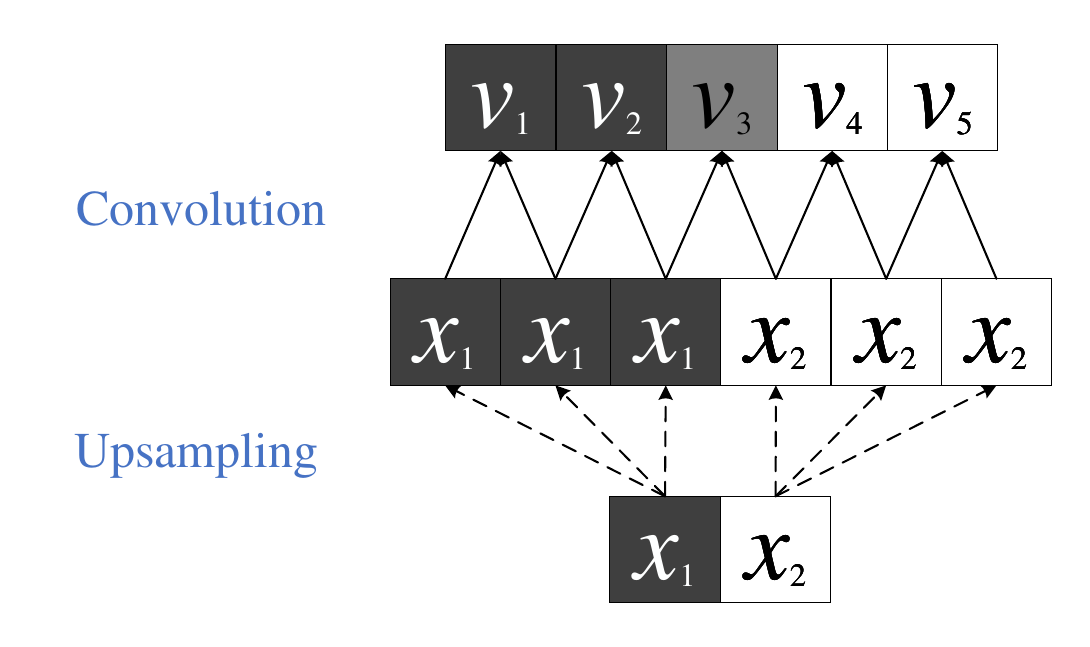}}
	\caption{
		(a) Upsampling the input using nearest interpolation with different scales.
		(b) An illustration of feature redundancy caused by unsampling.
	}
\end{figure}

A straightforward way is to upsample the input such that the subsequent feature maps are also scaled up.
Classical interpolators and learnable upsampling methods can be used here.
Specifically, nearest and bilinear interpolations are considered here
and Figure \ref{fig-upsampling} illustrates a case of nearest interpolation, where the input is upsampled by repeating the nearest point.
In addition, deconvolution \cite{zeiler2010deconv}, which is a fractionally-strided convolution with learnable filters, will also be compared in this paper.

Upsampling inputs may generate redundant feature entries so that only very local structures are convolved in the early stages.
A toy example shown in Figure \ref{fig-redundancy} can demonstrate this phenomenon.
For an input $ x \in \mathbb{R}^2 $, when we enlarge it by a factor of 3 and then apply a 2-dimensional filter,
we will obtain an output feature vector  $ v=(v_1,v_2,\ldots, v_5) $.
In this case, $ v_1=v_2 $ and $ v_4=v_5 $ always hold true no matter what value the input takes.
It means that although the output appears to be the length of 5, two of them are actually redundant.
A similar problem exists implicitly in other upsampling methods, which prevents global information from being timely captured by the model.
Admittedly, this redundancy issue may not reduce the final performance, but it raises concerns about efficiency.
The experiments in Section \ref{section-Up-verus-stride} will compare in detail the modifications introduced here and next.

\subsection{Shrinking sliding strides}

\begin{figure}
	\centering
	\centering
	\scalebox{0.4}{\includegraphics{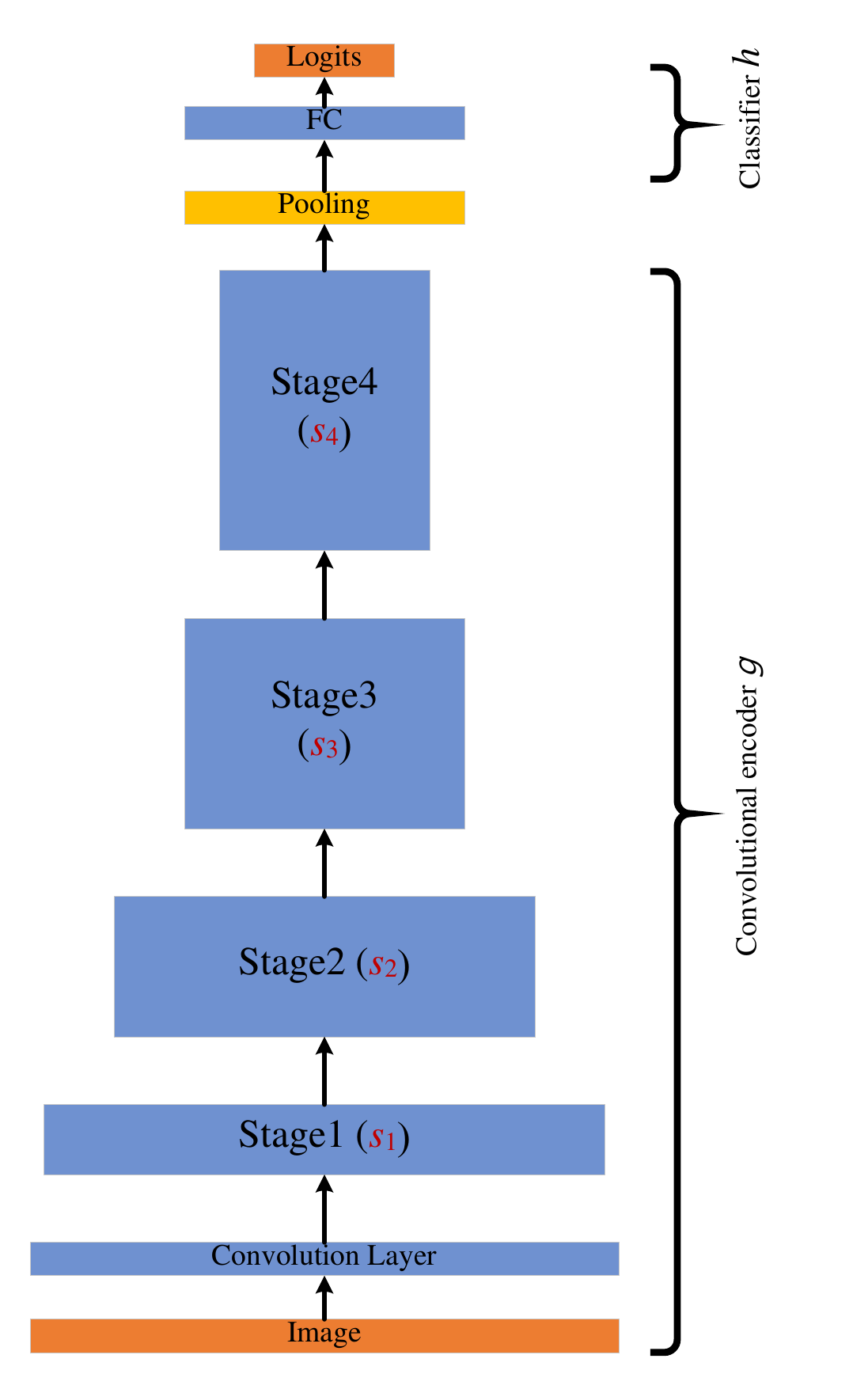}}
	\caption{
		The architecture of ResNet with the \textcolor{red}{$s_1$-$s_2$-$s_3$-$s_4$} sliding stride configuration.
		It consists of 2 fixed terminals and  4 intermediate stages. Each stage contains several convolutional layers.
	}
	\label{fig-framework-resnet18}
\end{figure}

Let us take a closer look at the convolution operation.
Convolving the filter $w \in \mathbb{R}^{k \times k}$ across an input $X \in \mathbb{R}^{\mathfrak{H} \times \mathfrak{W}}$
in a \textit{cross-correlation} fashion \cite{goodfellow2016deeplearning} can be defined as
\begin{align*}
	 [X \star w]_{ij} = \sum_{m=1}^k \sum_{n=1}^k w_{mn} X_{i+m - 1,j+n - 1}.
\end{align*}
The convolution performed in practice sometimes skips some positions for the purpose of downsampling.
Sampling every $s$ positions in each direction yields the following convolution operation
\begin{align*}
	[X \star w]_{ij}  = \sum_{m=1}^k \sum_{n=1}^k w_{mn} X_{(i - 1)s   + m, (j - 1)s  + n},
\end{align*}
which acts as a downsampling if the \textit{sliding stride} of $s$ is greater than 1.
Now the  corresponding output feature map sizes $ H $ and $ W $ are determined by
\begin{align}
\label{slideOut}
\begin{split}
	H &= \lfloor \frac{\mathfrak{H} + 2p - k}{s} + 1 \rfloor,
    \\[5pt]
	W &= \lfloor \frac{\mathfrak{W} + 2p - k}{s} + 1 \rfloor,
\end{split}
\end{align}
where $p$ is the padding that allows the input to be downsampled exactly to $1/s$.

It can be seen from Eq. \eqref{slideOut}
that shrinking the sliding stride $s$ (\eg, from $2$ to $1$) leads to larger (\eg, $2 \times$) convolutional feature maps,
producing an effect similar to upsampling the input.
In addition, this operation will not cause the feature redundancy issue mentioned earlier.

We take ResNet as an example to illustrate how to modify the stride configurations.
As shown in Figure \ref{fig-framework-resnet18},
the convolutional part of ResNet consists of 4 intermediate stages, each of which contains several convolutional layers.
The sliding stride configuration of the conventional ResNet is 1-2-2-2.
If it takes as input a $32 \times 32$ image, it would output $4 \times 4$ feature maps.
By changing this baseline configuration to 1-1-2-2, all feature maps after the second stage will be enlarged by a factor of 2,
resulting in the final larger sizes (\ie, $8 \times 8$ feature maps for a $32 \times 32$ input) before average pooling.
In addition, although the 1-1-2-2 and 1-2-1-2 stride configurations can lead to the same feature map sizes,
the configuration of 1-1-2-2 is still preferred for low-resolution datasets as it contributes to better model capacity (see Section \ref{section-Up-verus-stride}).
However, in the case of high resolution, the opposite may occur due to the slower perceptual field growth in the former (see Appendix \ref{section-high-resolution}).

It is worth mentioning that not all CNNs employ convolution for downsampling.
Some classical architectures such as AlexNet and VGG resort to max pooling that slides in the same manner as introduced above.
Needless to say, they can also be improved in the same way.

\begin{rmrk}
WideResNet \cite{zagoruyko2016wideresnet} has three stages with a stride configuration of 1-2-2 and can be seen as a wider version of the modified ResNet18 with the (1-1)-2-2 stride configuration.
In this sense, WideResNet's stride configuration is already optimal.
To further improve this network architecture, the strategy presented in \cite{huang2021wideresnetR} would be a feasible way.
\end{rmrk}

\section{Experiments}
\label{section-experiments}

In this section, we comprehensively investigate the effectiveness of the proposed modifications.
Firstly, we carefully compare in Section \ref{section-Up-verus-stride} the two feasible ways, upsampling inputs versus shrinking sliding strides.
More empirical analysis regarding the suggestions from Proposition~\ref{proposition-stability} is left in Section~\ref{section-empirical-analysis}.
Then, we extend the improvements to more classical neural network architectures, which are presented in Section~\ref{section-other-architectures}.
Section \ref{section-overall} further demonstrates that the proposed modifications are orthogonal to defense mechanisms and datasets.
Finally, gradient obfuscation concerns~\cite{athalye2018obfuscation} will be discussed in Section \ref{section-gradient-obfuscation}.

\subsection{Experimental setup}
\label{section-setup}

In this part, we introduce the datasets, baseline defenses, evaluation metrics, attacks, and architectures.

\begin{table}[htb]
	\caption{The Basic Setup for Adversarial Attacks in $\ell_{\infty}$ Norm.}
	\label{table-settings}
	\centering
	\scalebox{.8}{
		\begin{tabular}{cccccccc}
			\toprule
			& FGSM & PGD & PGD & PGD & DeepFool& AutoAttack & \\
			\midrule
			Number of iterations & - & 10 & 20 & 50 & 50  & - &  \\
			Step size & - & 0.25 & 0.25 &  0.033333 & 0.02& - & \\
			\bottomrule
	\end{tabular}}
\end{table}

\begin{figure*}
	\center
	\subfloat[Upsampling]{\label{fig-performance-scale}\includegraphics[height=0.18\textwidth]{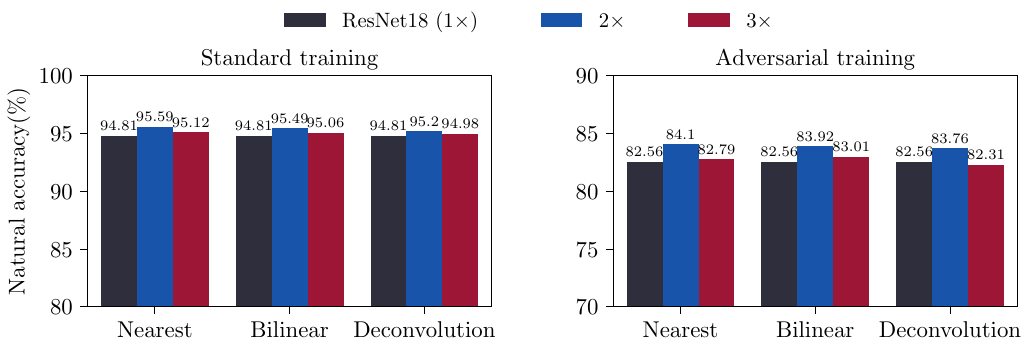}}
	\hspace{12pt}
	\subfloat[Stride configurations]{\label{fig-performance-stride}\includegraphics[height=0.18\textwidth]{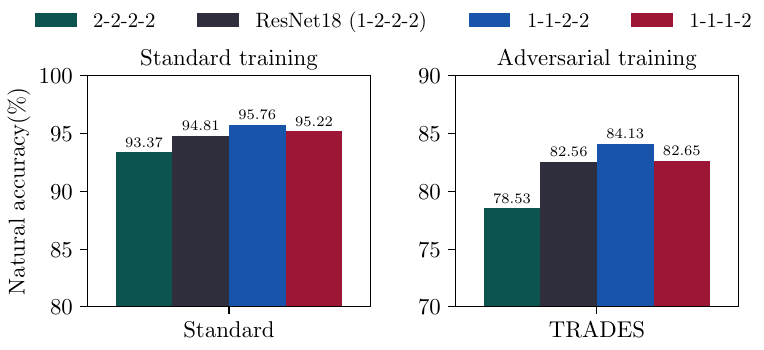}}
	\caption{
		Natural accuracy (\%) comparisons of ResNet18 and its modified versions on CIFAR-10.
		(a) Upsampling using bilinear and nearest interpolation, or learnable deconvolution under different scales;
		(b) Shrinking sliding strides using various configurations.
	}
	\label{fig-performance}
\end{figure*}

\textbf{Datasets.}
The widely-used CIFAR-10 and CIFAR-100 datasets \cite{krizhevsky2009cifar} are considered for robustness evaluation.
They contain 60,000 $32 \times 32$ real-world images in 10 and 100 classes, respectively.
In addition, CIFAR-10-C \cite{hendrycks2019cifarc}, consisting of nineteen different types of semantically invariant corruptions (\eg, blur and noise),
will be used to measure the resistance against common corruptions.
Since there is no requirement for hyperparameter tuning, the validation set will not be specially prepared.

\textbf{Baseline defenses.}
We implement the baselines including AT \cite{madry2018pgd}, ALP \cite{kannan2018alp}, TRADES \cite{zhang2019trades},
MART \cite{wang2020mart}, FAT \cite{zhang2020fat}, AWP \cite{wu2020awp} and DAJAT \cite{addepalli2020dajat},
identically following the settings suggested in the original papers.
All adversarial samples required during training are crafted on the fly by PGD-10 within the perturbation budget of $\epsilon=8/255$ (in $\ell_{\infty}$ norm).
In particular, we apply an early-stopping learning schedule~\cite{pang2020bag} to AT, ALP and TRADES,
namely, the SGD optimizer with an initialized learning rate of 0.1 which is decayed by a factor of 10 at 100 and 105 epochs.

\textbf{Evaluation metrics.}
The model capacity and anti-perturbation ability are the main concerns in this paper,
but they are difficult to analyze quantitatively.
We thus use the following metrics to estimate them:
\begin{itemize}[leftmargin=*]
\item \textit{Natural accuracy} is the proportion of clean samples that are correctly classified.
		In general, a neural network with superior model capacity can fit the data distribution better and thus enjoy higher natural accuracy \cite{cao2019}.
\item \textit{Consistent accuracy} is the proportion of clean samples whose classification results remain unchanged after perturbation.
		It measures the resistance to perturbation, \ie, anti-perturbation ability.
\item \textit{Adversarial robustness} measures the intersection of the two,
		which is the proportion of clean samples that are correctly and consistently classified within the given perturbation budget.
\end{itemize}

\textbf{Attacks.}
On the one hand, the exact adversarial robustness must be calculated by exhausting all possible perturbations, which is computationally prohibitive.
On the other hand, a specific attack may overestimate the robustness.
Hence, to reliably evaluate the adversarial robustness,
several benchmark attacks including
FGSM \cite{goodfellow2015}, PGD \cite{madry2018pgd}, DeepFool \cite{moosavi2016deepfool}
and AutoAttack \cite{croce2020aa} are employed.
In particular,
AutoAttack is one of the most elaborate and widely used evaluation metrics,
allowing for a reliable estimation of adversarial robustness.

Except that AutoAttack is due to the source code from \cite{croce2020aa},
all other implementations are provided by FoolBox \cite{foolbox2017}.
The major settings of these attacks are listed in Table \ref{table-settings},
wherein step size denotes the relative step size of PGD and the overshoot of DeepFool, respectively.

\textbf{Architectures.}
Besides the widely-used ResNet18 \cite{he2016resnet},
we also try to explore appropriate stride configurations for other CNNs
including AlexNet \cite{krizhevsky2012alex}, VGG \cite{simonyan2015vgg} and PreActResNet18~\cite{he2016preactresnet}
(see Appendix \ref{appendix-experimental-details} for architectural details).
In the following, we use a form similar to $s_1$-$s_2$-$s_3$-$s_4$ to represent different sliding stride configurations.

\subsection{Upsampling inputs versus shrinking sliding strides}
\label{section-Up-verus-stride}

\begin{figure}
	\centering
	\scalebox{.8}{\includegraphics{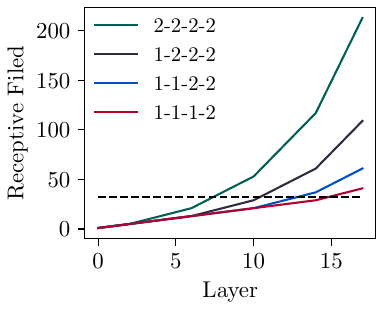}}
	\caption{
		Receptive field of ResNet18 with various stride configurations.
		The \textbf{horizonal dashed line} therein indicates the image size of 32.
	}
	\label{fig-strides-receptive-field}
\end{figure}

We first compare the effectiveness of the two proposed modifications in terms of model capacity and anti-perturbation ability.
And the last part concerning computational efficiency will justify why shrinking sliding strides is preferable to upsampling inputs.
For upsampling, we also consider a learnable method called deconvolution \cite{zeiler2010deconv},
which is a fractionally-strided convolution with learnable filters.
TRADES is adopted as the representative of adversarial training because it rarely kills the training in the first iterations.
In addition, unless otherwise stated, the robustness reported hereinafter is empirically evaluated by PGD-20.

\textbf{Model capacity.}
As can be seen in Figure \ref{fig-performance}, both modifications show similar results:
the natural accuracy of the modified ResNet18 is always better than the baseline, either by standard training or adversarial training.
Since natural accuracy is largely determined by model capacity,
it suggests that enlarging convolutional feature maps helps to promote model capacity.

However, this conclusion is true for scale factors up to 2.
Further increasing feature map sizes (marked in \textcolor{red}{red}) will lead to a decrease in performance.
There are two possible reasons for this phenomenon:
1) When the factor is large enough, the resulting neurons become redundant.
2) Excessive shrinkage of downsampling operators slows down the growth of the receptive field, which in turn deteriorates the learning ability of the network.
As shown in Figure \ref{fig-strides-receptive-field}, ResNet18 with a 1-1-1-2 stride configuration cannot capture global information until the 16-th layer.

\begin{table}[htb]
	\caption{
		The \underline{Consistent Accuracy} (\%) on CIFAR-10 under various perturbation budgets $\epsilon$.
		\textbf{Top}: Upsampling inputs using interpolation or learnable methods;
		\textbf{Bottom}: Shrinking sliding strides using various configurations.
	}
	\label{table-stability}
	\centering
	\scalebox{.85}{
		\begin{tabular}{c|cccccc}
			\toprule
			& \multicolumn{6}{c}{Consistent Accuracy ($\epsilon, \ell_{\infty}$)} \\
			\midrule
		Modifications &  $0/255$  &  $2/255$  &  $4/255$   & $6/255$   & $8/255$   &  $10/255$  \\
			\midrule
		ResNet18 ($1\times$)&100&86.54&75.66&65.48&55.67&45.91\\
		Nearest ($2\times$)&100&\textbf{87.59}&\textbf{76.97}&66.72&\textbf{57.11}&47.30\\
		Nearest ($3\times$)&100&87.17&75.62&65.60&55.71&46.82\\
		Bilinear ($2\times$)&100& 87.35 & 76.64 & 66.19 & 56.52 & 46.95  \\
		Bilinear ($3\times$)&100&86.76&76.29&65.81&55.99&46.53 \\
		Deconvolution ($2\times$)&100& 87.27 & 76.78 & \textbf{66.88} & 56.73 & \textbf{47.49}  \\
		Deconvolution ($3\times$)&100& 86.68 & 75.46 & 65.66 & 55.84 & 46.73  \\
		\midrule
		2-2-2-2  &100&85.05&73.19&61.80&51.61&42.31\\
		2-1-2-2  &100& 86.99 & 75.23 & 64.35 & 54.68 & 45.27 \\
		ResNet18(1-2-2-2)&100&86.54&75.66&65.48&55.67&45.91\\
		1-2-2-1&100& 86.80 & 75.96 & 65.51 & 55.72 & 46.10 \\
		1-2-1-2&100& 87.41 & 76.67 & 66.11 & 56.51 & 46.82 \\
		1-1-2-2&100&\textbf{87.78}&\textbf{76.86}&\textbf{66.39}&\textbf{57.11}&\textbf{47.44}\\
		1-1-1-2&100&86.92&76.16&65.82&56.59&47.26\\
			\bottomrule
		\end{tabular}
	}
\end{table}

\begin{figure}
	\centering
	\scalebox{.7}{\includegraphics{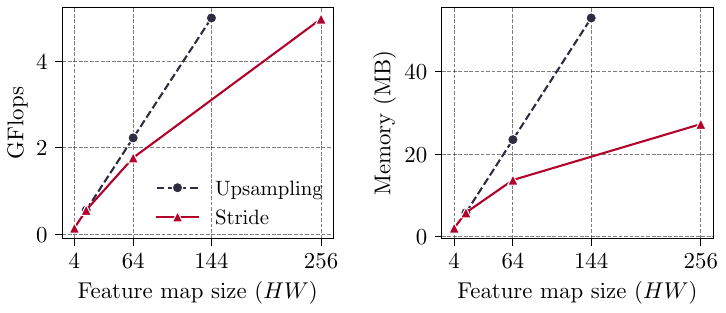}}
	\caption{
		Computational costs of the inference on a $32\times 32$ image.
		\textbf{Left}: Floating point operations (Flops);
		\textbf{Right}: Total memory.
	}
	\label{fig-computational-cost}
\end{figure}

\begin{figure*}
	\centering
	\scalebox{0.7}{\includegraphics{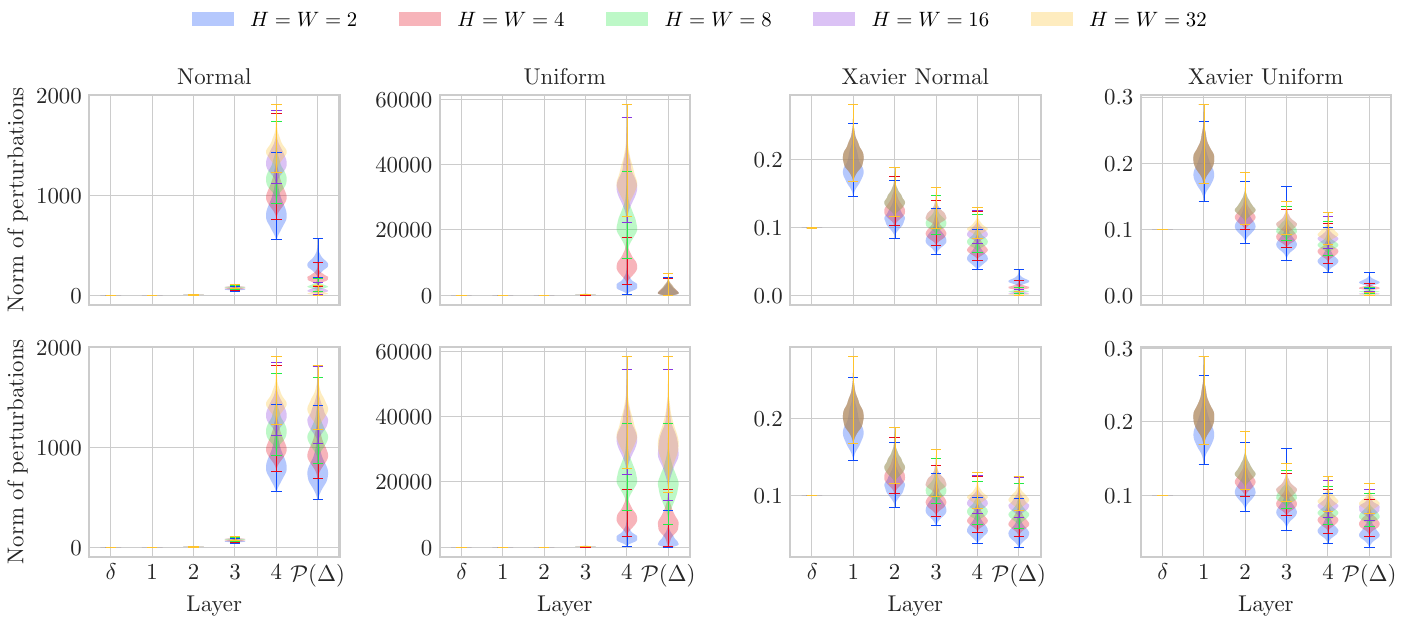}}
	\caption{
		The $\ell_{\infty}$ norm of input perturbations from $\mathcal{U}[-0.1, 0.1]$ and the resulting changes at each layer.
		Different initialization methods, including Normal, Uniform, Xavier Normal and Xavier Uniform \cite{glorot2010xavier}, are used to eliminate the effect of randomness.
		\textbf{Top:} A randomly initialized CNN followed by average pooling;
		\textbf{Bottom:} A randomly initialized CNN followed by max pooling.
	}
	\label{fig-empirical-random}
\end{figure*}

\textbf{Anti-perturbation ability.}
We use consistent accuracy, the proportion of clean samples whose classification results remain unchanged after perturbation,
to estimate the anti-perturbation ability of networks.

Table \ref{table-stability} presents the performance across various modified RestNet versions,
among which the best results are achieved in the case of enlarging convolutional feature maps by a factor of 2.
Note that the 1-2-2-1 and 1-2-1-2 stride configurations are not as effective as the 1-1-2-2 configuration,
because the latter, with the low-resolution dataset, contributes better to model capacity (see Appendix \ref{section-high-resolution} for a discussion concerning the high-resolution case).

Although a further increase in size hurts the performance slightly, this degradation in consistency accuracy is much more acceptable than that in natural accuracy.
As can be seen from Figure \ref{fig-performance-stride}, ResNet18 with the 1-1-1-2 stride configuration performs comparably to the baseline (marked in \textbf{black}) on natural accuracy.
If the consistent accuracy benefits primarily from the model capacity, a similar conclusion should be drawn from Table \ref{table-stability}.
However, ResNet18 with the 1-1-1-2 stride configuration significantly surpasses the baseline, even on par with the optimal configuration of 1-1-2-2.
Hence, we can infer that the consistent accuracy improvements benefit from the enhanced anti-perturbation ability.
It is worth noting that upsampling does not have similar properties, especially when high perturbation budgets are considered.
This corroborates what was discussed in Section \ref{section-upsampling},
where upsampling may lead to feature redundancy and thus weakens the learning ability of the model.

\textbf{Computational efficiency.}
Figure \ref{fig-computational-cost} compares the floating point operations (Flops)
and the total memory required for the two modifications.
Upsampling inputs consumes more theoretical Flops and memory on meaningless redundant elements.
Notably, this computational gap will widen further in practice,
especially in adversarial training scenarios.
So shrinking sliding strides is preferable.
In addition, some acceleration methods (\eg, FastAT \cite{wong2020fgsmrs})
can also be applied to the improved architectures for acceptable training costs.

According to the above three groups of experiments, we can conclude that
1) properly enlarging features are capable of improving model capacity and anti-perturbation ability;
2) shrinking sliding strides is more effective than upsampling inputs.

\subsection{Average pooling versus max pooling}
\label{section-empirical-analysis}

Proposition \ref{proposition-stability} suggests that average pooling is a superior choice to max pooling.
Here, we compare them on random noise and real-world data for a complete understanding.

\textbf{Random noise.}
A randomly initialized CNN,
which contains 4 convolutional layers (with 3, 16, 32 and 64 channels, respectively) followed by average pooling or max pooling,  is considered here.
It takes a zero-valued image as input; in the absence of perturbations, the subsequent feature maps should be zero-valued as well.
As such, the norm of the changes resulting from the input perturbations indicates the anti-perturbation ability: the lower the better.
For a fair comparison, each subplot in Figure \ref{fig-empirical-random} is drawn based on the same 1000 uniform random noises.
Although some assumptions in Proposition \ref{proposition-stability} can not be strictly satisfied,
the observed phenomenon is still consistent with our analysis in Section~\ref{section-theory}:
1) The resulting changes after average pooling are fairly small, and become smaller as the feature size increases.
2) Conversely, max pooling has little effect on mitigating changes regardless of the weight initialization method.

\begin{table}[tb]
	\caption{
	Natural accuracy (\%) and adversarial robustness (\%) on CIFAR-10 using ResNet18 with Average Pooling or Max Pooling.
	}
	\label{table-avg-max}
	\centering
	\scalebox{.9}{
		\begin{tabular}{c||c|cc}
			\toprule
			Configuration & Pooling Type & Natural & PGD-20 \\ 
			\midrule
			\multirow{2}{*}{2-2-2-2} & Average & 78.53 & 49.05  \\ 
			 & Max & 79.36 & 47.30 \\ 
			\midrule
			\multirow{2}{*}{ResNet18 (1-2-2-2)} & Average &82.56& 53.50  \\ 
			 & Max &81.49&53.29 \\ 
			\midrule
			\multirow{2}{*}{1-1-2-2} & Average &\textbf{84.13}&\textbf{55.37} \\ 
			 & Max &81.12&52.87 \\ 
			\midrule
			\multirow{2}{*}{1-1-1-2} & Average &82.65&54.96 \\ 
			 & Max & 78.59 & 50.17  \\ 
			\bottomrule
	\end{tabular}}
\end{table}

\begin{figure}
	\centering
	\scalebox{.5}{\includegraphics{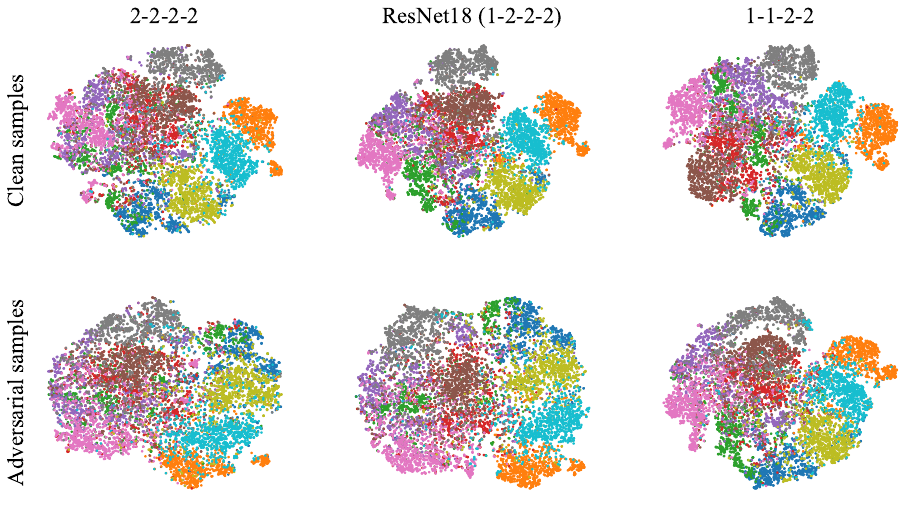}}
	\caption{
		t-SNE \cite{tsne2009van} visualization of features after average pooling.
		Different classes are represented by different colors.
		\textbf{Top:} Clean CIFAR-10 images.
		\textbf{Bottom:} Adversarial CIFAR-10 images (perturbed by PGD-10).
	}
	\label{fig-tsne}
\end{figure}

\textbf{Real-world data.}
Table \ref{table-avg-max} presents the natural accuracy and adversarial robustness evaluated on real-world data, CIFAR-10.
For the baseline and 2-2-2-2 stride configurations, max pooling and average pooling can be used interchangeably.
However, as the size increases, the performance of max pooling decreases significantly.
It corroborates that smaller feature map sizes are more suitable for max pooling.
It is worth noting that for different reasons, max pooling performs poorly with the 2-2-2-2 and 1-1-1-2 stride configurations.
The former is due to the lack of model capacity,
while the latter is mainly caused by the mismatch between max pooling and large feature map sizes.
Furthermore, Figure \ref{fig-tsne} illustrates the t-SNE visualization of feature representations across different stride configurations.
It can be seen that larger feature map sizes lead to slightly better separability, especially for adversarial samples.

From the above comparisons, we can conclude that average pooling is preferable to max pooling because it promotes anti-perturbation ability theoretically and experimentally.

\subsection{Improvements on various neural network architectures}
\label{section-other-architectures}

\begin{table}[htb]
	\caption{
		Natural Accuracy (\%) and Adversarial Robustness (\%) under PGD-20 attack of various neural network architectures on CIFAR-10.		
		The \underline{underline} indicates the modified configuration.
	}
	\label{table-architectures}
	\centering
	\scalebox{.9}{
		\begin{tabular}{c||c|ccc}
			\toprule
			& Configuration & Natural	&  PGD-20 \\
			\midrule
			\multirow{2}{*}{AlexNet} & 2-2-2-2 &70.56&36.57 \\
			   &\underline{1-2-2-2}&\textbf{77.34}&\textbf{41.97} \\
			\midrule
			\multirow{2}{*}{VGG16} & 2-2-2-2-2 &79.61&43.28 \\
			   &\underline{1-1-2-2-2}&\textbf{84.47}&\textbf{48.61} \\
			\midrule
			\multirow{2}{*}{ResNet18} & 1-2-2-2 &82.56&53.50 \\
			   &\underline{1-1-2-2}&\textbf{84.13}&\textbf{55.37} \\
			\midrule
			\multirow{2}{*}{PreActResNet18} & 1-2-2-2 &82.58&53.05\\
			   &\underline{1-1-2-2}&\textbf{83.92}&\textbf{54.64}\\
			\bottomrule
	\end{tabular}}
\end{table}

\textbf{Classical CNNs.}
We extend the sliding stride modification to more classical CNNs.
Table \ref{table-architectures} shows that properly modifying the stride configuration can contribute to the performance of various network architectures.
Note that even though AlexNet and VGG16 employ max pooling by default,
they still enjoy consistent accuracy and robustness gains from the modified configuration.
In addition, the performance gains of AlexNet and VGG16 far exceed those of ResNet18,
mainly because their original configurations contain more unsatisfactory downsampling operations,
so greater gains can be obtained after shrinking the sliding stride.

\begin{table}[tb]
	\caption{
		Natural accuracy (\%) and adversarial robustness (\%) under PGD-20 attack of ViT-tiny on CIFAR-10.
	}
	\label{table-vit}
	\centering

	\scalebox{.9}{
	\begin{tabular}{c|c|cc}
			\toprule
			 & Pooling Type  & Natural	&  PGD-20 \\
			\midrule
			\multirow{2}{*}{ ViT-tiny } & Average & 66.97 & 34.15 \\
			& Max & 65.58 & 33.03 \\
			\midrule
			\multirow{2}{*}{Half patch size} & Average &72.69&36.12 \\
			& Max & 68.40 & 31.95 \\
			\midrule
			\multirow{2}{*}{Nearest ($2\times$)} & Average & \textbf{73.16} & \textbf{37.45} \\
			& Max & 69.06 & 34.03 \\
			\bottomrule
	\end{tabular}}

\end{table}

\textbf{ViTs.}
The success of transformers in NLP leads to an increasing interest in vision transformers.
For example, ViTs~\cite{dosovitskiy2021vits} project each patch of the image linearly into an embedding,
and the subsequent attention modules transform them into discriminative features.
Although ViTs are beyond the scope of CNNs, it is possible to enlarge the feature maps by upsampling or using a smaller patch size.
As shown in Table \ref{table-vit}, the baseline ViT-tiny adopts the patch size of 8 with averaging pooling.
Reducing the patch size to half or upsampling the input image yields $2\times$ feature maps.
We can find that larger feature maps before average pooling can contribute to better resistance to perturbations,
but the conclusion is not true for max pooling.
It is consistent with the conclusion in terms of CNNs.

\begin{table*}[t]
	\center
	\caption{
	Natural accuracy (\%) and robustness (\%) under various attacks within the perturbation budget $\epsilon=8/255$.	
	$\dagger$: Applying the learning schedule \cite{pang2020bag} for early stopping;
	$\ddagger$: The best result on the checkpoint evaluated by PGD-10.
	\underline{Underline} indicates the modified ResNet18.
	}
	\label{table-cifar-linf}
	\scalebox{.9}{
		\begin{tabular}{ccccccccccccc}
			\toprule
		& \multicolumn{2}{c}{Natural}	&  \multicolumn{2}{c}{FGSM} &  \multicolumn{2}{c}{PGD-20}	&  \multicolumn{2}{c}{PGD-50}	&  \multicolumn{2}{c}{DeepFool}	&  \multicolumn{2}{c}{AutoAttack} \\
		\cmidrule{2-3} \cmidrule{4-5} \cmidrule{6-7} \cmidrule{8-9} \cmidrule{10-11} \cmidrule{12-13}
		& 1-2-2-2 & \underline{1-1-2-2} & 1-2-2-2 & \underline{1-1-2-2} & 1-2-2-2 & \underline{1-1-2-2} & 1-2-2-2 & \underline{1-1-2-2} & 1-2-2-2 & \underline{1-1-2-2} & 1-2-2-2 & \underline{1-1-2-2} \\
			\midrule
\multicolumn{13}{c}{CIFAR-10} \\
			\midrule
$\text{AT}^{\dagger}$ \cite{madry2018pgd}          &84.15&86.17 \textcolor{green}{$\uparrow$} &58.85&60.70 \textcolor{green}{$\uparrow$} &52.80&53.24 \textcolor{green}{$\uparrow$} &52.50&53.18 \textcolor{green}{$\uparrow$} &54.32&55.90 \textcolor{green}{$\uparrow$} &48.72&49.83 \textcolor{green}{$\uparrow$}\\				
$\text{ALP}^{\dagger}$ \cite{kannan2018alp}        &86.53&88.11 \textcolor{green}{$\uparrow$} &58.73&59.21 \textcolor{green}{$\uparrow$} &50.88&51.36 \textcolor{green}{$\uparrow$} &50.76&51.27 \textcolor{green}{$\uparrow$} &54.47&55.39 \textcolor{green}{$\uparrow$} &47.38&49.00 \textcolor{green}{$\uparrow$}\\				
$\text{TRADES}^{\dagger}$ \cite{zhang2019trades}     &82.56&84.13 \textcolor{green}{$\uparrow$} &58.06&59.60 \textcolor{green}{$\uparrow$} &53.50&55.37 \textcolor{green}{$\uparrow$} &53.37&55.28 \textcolor{green}{$\uparrow$} &54.50&56.05 \textcolor{green}{$\uparrow$} &49.71&51.24 \textcolor{green}{$\uparrow$}\\				
MART            \cite{wang2020mart}               &82.77&84.67  \textcolor{green}{$\uparrow$} &59.28&60.46 \textcolor{green}{$\uparrow$} &52.44&53.00 \textcolor{green}{$\uparrow$} &52.28&52.91 \textcolor{green}{$\uparrow$}&54.07&54.58 \textcolor{green}{$\uparrow$} &47.29&47.97 \textcolor{green}{$\uparrow$}\\				
FAT         \cite{zhang2020fat}                   &87.34&88.31 \textcolor{green}{$\uparrow$}&55.28&53.43 \textcolor{red}{$\downarrow$}&43.90&42.55 \textcolor{red}{$\downarrow$}&43.79&42.39 \textcolor{red}{$\downarrow$}&51.20&50.04 \textcolor{red}{$\downarrow$}&41.56&40.93 \textcolor{red}{$\downarrow$}\\				
$\text{FAT}^{\ddagger}$  \cite{zhang2020fat}      &87.55&\textbf{88.86} \textcolor{green}{$\uparrow$}&56.38&57.42 \textcolor{green}{$\uparrow$}&46.67&46.91 \textcolor{green}{$\uparrow$}&46.56&46.83 \textcolor{green}{$\uparrow$}&52.54&53.39 \textcolor{green}{$\uparrow$}&44.10&44.95  \textcolor{green}{$\uparrow$}\\				
AT-AWP            \cite{wu2020awp}             &82.32&84.05 \textcolor{green}{$\uparrow$}&59.47&60.41 \textcolor{green}{$\uparrow$}&54.47&54.75 \textcolor{green}{$\uparrow$}&54.40&54.59 \textcolor{green}{$\uparrow$}&54.63&55.30 \textcolor{green}{$\uparrow$}&49.52&50.23  \textcolor{green}{$\uparrow$}\\				
TRADES-AWP        \cite{wu2020awp}            &82.69&84.36 \textcolor{green}{$\uparrow$}&59.96&61.39 \textcolor{green}{$\uparrow$}&55.90&57.14 \textcolor{green}{$\uparrow$}&55.88&57.10 \textcolor{green}{$\uparrow$}&55.94&57.65 \textcolor{green}{$\uparrow$}&51.84&53.51 \textcolor{green}{$\uparrow$}\\				
DAJAT \cite{addepalli2020dajat}   &84.45& 85.96 \textcolor{green}{$\uparrow$}& 61.88 & \textbf{63.24}  \textcolor{green}{$\uparrow$}& 57.61 & \textbf{58.55} \textcolor{green}{$\uparrow$}& 59.49 & \textbf{60.70} \textcolor{green}{$\uparrow$}& 57.37 & \textbf{59.01}  \textcolor{green}{$\uparrow$}& 52.76 & \textbf{54.14}  \textcolor{green}{$\uparrow$} \\
\midrule
\multicolumn{13}{c}{CIFAR-100} \\
\midrule
$\text{AT}^{\dagger}$   \cite{madry2018pgd}    &59.16&62.50 \textcolor{green}{$\uparrow$}&32.95&35.30 \textcolor{green}{$\uparrow$}&28.88&31.15 \textcolor{green}{$\uparrow$}&28.87&31.05 \textcolor{green}{$\uparrow$}&28.32&30.24 \textcolor{green}{$\uparrow$}&25.16&27.04  \textcolor{green}{$\uparrow$} \\
$\text{ALP}^{\dagger}$ 	\cite{kannan2018alp}    &63.41&66.70 \textcolor{green}{$\uparrow$}&30.45&32.05 \textcolor{green}{$\uparrow$}&25.33&26.81 \textcolor{green}{$\uparrow$}&25.34&26.70 \textcolor{green}{$\uparrow$}&26.51&27.65 \textcolor{green}{$\uparrow$}&23.13&24.06  \textcolor{green}{$\uparrow$} \\
$\text{TRADES}^{\dagger}$ \cite{zhang2019trades}  &58.69&59.05 \textcolor{green}{$\uparrow$}&32.59&33.66 \textcolor{green}{$\uparrow$}&30.22&31.27 \textcolor{green}{$\uparrow$}&30.17&31.21 \textcolor{green}{$\uparrow$}&28.06&27.97 \textcolor{red}{$\downarrow$}&25.37&25.73 \textcolor{green}{$\uparrow$}  \\
MART                  \cite{wang2020mart}      &54.74&56.01 \textcolor{green}{$\uparrow$}&34.05&35.02 \textcolor{green}{$\uparrow$}&31.67&32.90 \textcolor{green}{$\uparrow$}&31.70&32.87 \textcolor{green}{$\uparrow$}&28.42&29.04 \textcolor{green}{$\uparrow$}&26.24&26.59  \textcolor{green}{$\uparrow$}  \\
FAT                   \cite{zhang2020fat}      &62.23&64.72 \textcolor{green}{$\uparrow$}&26.75&26.86 \textcolor{green}{$\uparrow$}&20.01&20.20 \textcolor{green}{$\uparrow$}&20.04&20.11 \textcolor{green}{$\uparrow$}&23.29&23.42 \textcolor{green}{$\uparrow$}&18.89&18.65  \textcolor{red}{$\downarrow$} \\
$\text{FAT}^{\ddagger}$  \cite{zhang2020fat}   &63.69&66.40 \textcolor{green}{$\uparrow$}&28.19&29.17 \textcolor{green}{$\uparrow$}&22.01&22.68 \textcolor{green}{$\uparrow$}&21.90&22.69 \textcolor{green}{$\uparrow$}&24.48&25.17 \textcolor{green}{$\uparrow$}&20.05&20.82 \textcolor{green}{$\uparrow$} \\
AT-AWP                \cite{wu2020awp}      &58.47&61.37 \textcolor{green}{$\uparrow$}&35.37&37.22 \textcolor{green}{$\uparrow$}&32.65&34.23 \textcolor{green}{$\uparrow$}&32.60&34.13 \textcolor{green}{$\uparrow$}&30.12&31.65 \textcolor{green}{$\uparrow$}&27.47& \textbf{28.64}  \textcolor{green}{$\uparrow$} \\
TRADES-AWP            \cite{wu2020awp}      &59.53&60.81 \textcolor{green}{$\uparrow$}&34.23&35.22 \textcolor{green}{$\uparrow$}&31.96&32.59 \textcolor{green}{$\uparrow$}&31.87&32.50 \textcolor{green}{$\uparrow$}&29.34&29.55 \textcolor{green}{$\uparrow$}&26.57&26.94 \textcolor{green}{$\uparrow$} \\
DAJAT \cite{addepalli2020dajat}   &65.46& \textbf{67.65} \textcolor{green}{$\uparrow$}& 37.08 & \textbf{38.72}  \textcolor{green}{$\uparrow$}& 33.49 & \textbf{34.81} \textcolor{green}{$\uparrow$}& 35.15 & \textbf{36.46} \textcolor{green}{$\uparrow$}& 31.20 & \textbf{31.69}  \textcolor{green}{$\uparrow$}& 27.42 & 27.60  \textcolor{green}{$\uparrow$} \\

\bottomrule
\end{tabular}}
\end{table*}

\subsection{Combined with various defensive methods}
\label{section-overall}

In this section, we demonstrate that the proposed modifications are orthogonal to most of the defense mechanisms and
could help achieve further robustness improvements.
We take ResNet18 as the backbone network architecture,  with the default 1-2-2-2 stride configuration.
The modified architecture adopts the 1-1-2-2 stride configuration, thus enlarging the convolutional feature maps by a factor of 2.

\textbf{CIFAR-10/100 performance.}
Table~\ref{table-cifar-linf} reports the natural accuracy and adversarial robustness on CIFAR-10 and CIFAR-100.
According to the overall comparisons, we have the following key observations:
1) Most of the defense methods,
from the simplest AT, ALP to the somewhat complex AWP, DAJAT,
enjoy impressive improvements in both natural accuracy and adversarial robustness.
In view of the fact that there is no hyperparameter tuning for the modified ResNet18,
the proposed modifications undoubtedly contribute to performance.
2) The poor performance of FAT is due to the well-known over-fitting problem \cite{rice2020overfitting}.
Enlarging convolutional feature maps enhances the model capacity, which further exacerbates this situation.
As a result, FAT trained by the official setting fails to achieve optimal performance.
Stopping the training process early makes the modifications work again (see the results of $\text{FAT}^{\ddagger}$).

More interesting findings can be observed if we take a closer look at the results evaluated by the most aggressive attack, AutoAttack.
1) AT-AWP is an advanced version of AT, and the tricks therein result in a 0.80\% increase in adversarial robustness on CIFAR-10 but a 1.83\% decrease in natural accuracy.
This is an unsatisfactory trade-off encountered by most defense mechanisms.
In contrast, the modified ResNet18 can bring consistent improvements in both natural accuracy (2.02\% for AT) and adversarial robustness (1.11\% for AT).
This is due to the rationality of the proposed modifications, which can simultaneously enhance the model capacity and anti-perturbation ability.
2)
The state-of-the-art defense mechanisms can also benefit from the modifications.
For example, TRADES-AWP achieves a robustness gain of 1.66\% on CIFAR-10, and AT-AWP achieves a gain of 1.17\% on CIFAR-100.
It suggests that improving existing neural networks is a promising direction for significant and consistent robustness gains.

\textbf{Statistical significance.}
To verify the statistical significance,
we perform the t-test using the results of five independent runs for each baseline defensive method on CIFAR-10.
As can be seen from Figure~\ref{fig-significance},
except the adversarial robustness of $\text{FAT}^{\ddagger}$,
the  improvements on all baselines are statistically significant under $p$-value $<5\%$.

\begin{figure}[t]
	\centering
	\scalebox{.55}{\includegraphics{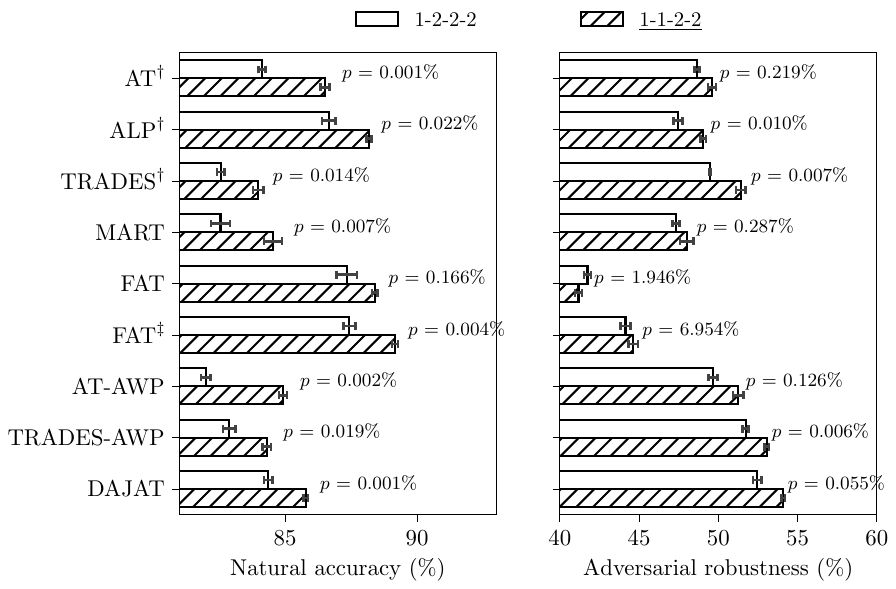}}
	\caption{
		Statistical significance of
		natural accuracy (\%) and adversarial robustness (\%) under AutoAttack.
		A paired t-test over 5 independent experiments is performed for each defense mechanism.
		\underline{Underline} indicates the improved ResNet18.
	}
	\label{fig-significance}
\end{figure}

\begin{figure*}[h]
	\centering
	\scalebox{.7}{\includegraphics{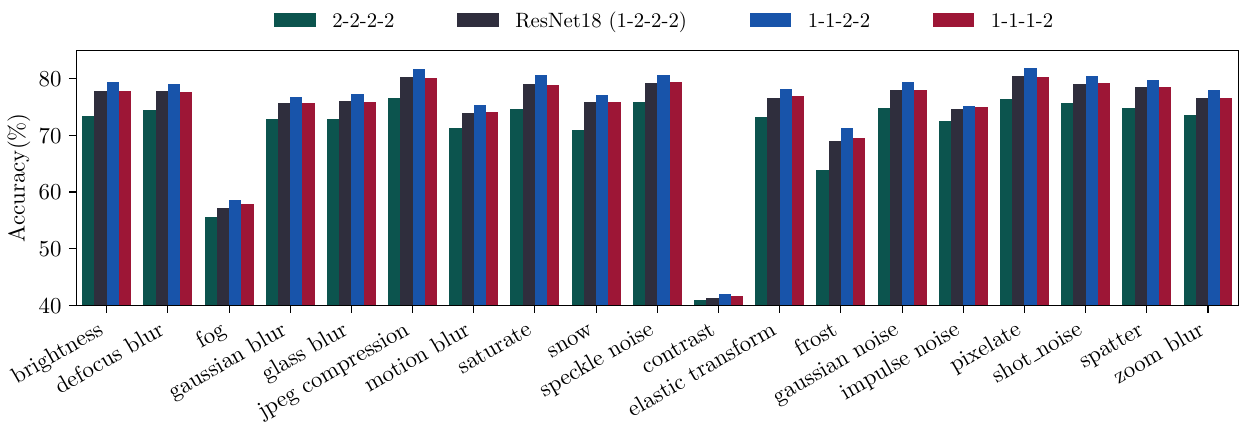}}
	\caption{
		Classification accuracy (\%) on CIFAR-10-C \cite{hendrycks2019cifarc}.
		It measures the general robustness against 19 different types of semantically invariant corruptions.
	}
	\label{fig-corruptions}
\end{figure*}

\textbf{CIFAR-10-C.}
We further use the CIFAR-10-C dataset \cite{hendrycks2019cifarc} to evaluate the general robustness against common corruptions.
It can be observed from Figure \ref{fig-corruptions} that due to the larger feature maps,
the proposed modification also improves the resistance to unseen semantically invariant corruptions.

\subsection{Gradient obfuscation concerns}
\label{section-gradient-obfuscation}

Layer transformation may raise concerns about gradient obfuscation \cite{athalye2018obfuscation}.
We show that this problem does not exist with the proposed modifications.
It can be first verified by Figure~\ref{fig-ablation-epsilon},
in which the robustness of the modified ResNet18 can be successfully dropped to zero as the budget increases.

\begin{figure}[hbt]
	\centering
	\scalebox{.8}{\includegraphics{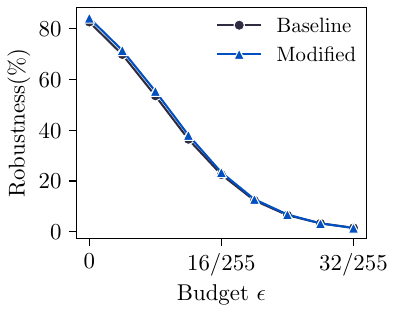}}
	\caption{
		Adversarial robustness comparison as the perturbation budget $\epsilon$ increases.
		The robustness of both the baseline and the improved variant gradually decreases to zero.
	}
	\label{fig-ablation-epsilon}
\end{figure}

In addition, we also design an adaptive attack (named LayerAttack) to craft the perturbations such that:
\begin{align*}
	\delta = \mathop{\text{argmax}} \limits_{\|\delta\|_{\infty} \le \epsilon} \: \|\Delta\|_2^2,
\end{align*}
or
\begin{align*}
	\delta = \mathop{\text{argmax}} \limits_{\|\delta\|_{\infty} \le \epsilon} \: \|\mathcal{P}(\Delta)\|_2^2.
\end{align*}
The corresponding results within the perturbation budget $\epsilon = 8 / 255$ are presented in Table \ref{table-layerattack}.
Of course, using $\|\Delta\|_2^2$ or $\|\mathcal{P}(\Delta)\|_2^2$ as a regularization term for the cross-entropy loss can encourage the perturbations to deceive the model;
however, we find that the optimal weight of this regularization term is close to zero, in which case it actually degrades to a normal PGD attack.
This can be explained by the fact that the proposed improvements do not suffer from gradient confusion.

\begin{table}[tb]
	\caption{
		Adversarial robustness evaluated by LayerAttack based on feature maps $\Delta$ and pooled feature maps $\mathcal{P}(\Delta)$.
	}
	\label{table-layerattack}
	\centering
	\scalebox{.9}{
		\begin{tabular}{cccccc}
			\toprule
			& & \multicolumn{2}{c}{LayerAttack-20}	&  \multicolumn{2}{c}{LayerAttack-50}  \\
			\cmidrule{3-6}
			Configuration & Natural  & $\Delta$	&  $\mathcal{P}(\Delta)$ & $\Delta$	&  $\mathcal{P}(\Delta)$ \\
			\midrule
			2-2-2-2 & 78.53 & 68.32 & 71.32 & 67.90 & 70.86 \\
			ResNet18 (1-2-2-2) & 82.56 & 73.24 & 74.93 & 72.89 & 74.46 \\
			1-1-2-2 & \textbf{84.13} & 76.15 & \textbf{77.34} & 75.71 & \textbf{77.12} \\
			1-1-1-2 & 82.65 & \textbf{76.86} & 76.44 & \textbf{76.66} & 76.16 \\
			\bottomrule
	\end{tabular}}
\end{table}

\section{Conclusion}

In this work, we theoretically analyze the relationship between the convolutional features and the perturbation resistance of the network.
It is discovered that averaging pooling in conjunction with large feature maps is capable of boosting anti-perturbation ability,
but the opposite is true for max pooling.
Based on the theoretical findings,
we propose two approaches to enlarge convolutional feature maps,
which greatly enhance the performance of existing CNNs,
even those using max pooling by default.
Our research helps expand the understanding of network robustness
and brings new inspiration to the design of robust deep learning models.

However, enlarging feature maps inevitably cause the degradation of the receptive field, which prevents the model from capturing global information timely.
More profound work needs to be developed from the perspective of convolutional feature maps and receptive fields.

\bibliographystyle{IEEEtran}
\bibliography{refs.bib}

\begin{thebibliography}{10}
\providecommand{\url}[1]{#1}
\csname url@samestyle\endcsname
\providecommand{\newblock}{\relax}
\providecommand{\bibinfo}[2]{#2}
\providecommand{\BIBentrySTDinterwordspacing}{\spaceskip=0pt\relax}
\providecommand{\BIBentryALTinterwordstretchfactor}{4}
\providecommand{\BIBentryALTinterwordspacing}{\spaceskip=\fontdimen2\font plus
\BIBentryALTinterwordstretchfactor\fontdimen3\font minus
  \fontdimen4\font\relax}
\providecommand{\BIBforeignlanguage}[2]{{%
\expandafter\ifx\csname l@#1\endcsname\relax
\typeout{** WARNING: IEEEtran.bst: No hyphenation pattern has been}%
\typeout{** loaded for the language `#1'. Using the pattern for}%
\typeout{** the default language instead.}%
\else
\language=\csname l@#1\endcsname
\fi
#2}}
\providecommand{\BIBdecl}{\relax}
\BIBdecl

\bibitem{wu2023aggn}
P.~Wu, Z.~Wang, B.~Zheng, H.~Li, F.~E. Alsaadi, and N.~Zeng, ``{AGGN:}
  attention-based glioma grading network with multi-scale feature extraction
  and multi-modal information fusion,'' \emph{Computers in Biology and
  Medicine}, vol. 152, p. 106457, 2023.

\bibitem{wu2023kdpar}
P.~Wu, Z.~Wang, H.~Li, and N.~Zeng, ``Kd-par: A knowledge distillation-based
  pedestrian attribute recognition model with multi-label mixed feature
  learning network,'' \emph{Expert Systems with Applications}, vol. 237, p.
  121305, 2024.

\bibitem{goodfellow2015}
I.~J. Goodfellow, J.~Shlens, and C.~Szegedy, ``Explaining and harnessing
  adversarial examples,'' in \emph{International Conference on Learning
  Representations (ICLR)}, 2015, pp. 1--11.

\bibitem{szegedy2013}
C.~Szegedy, W.~Zaremba, I.~Sutskever, J.~Bruna, D.~Erhan, I.~J. Goodfellow, and
  R.~Fergus, ``Intriguing properties of neural networks,'' in
  \emph{International Conference on Learning Representations (ICLR)}, 2014, pp.
  1--10.

\bibitem{moosavi2016deepfool}
S.~Moosavi{-}Dezfooli, A.~Fawzi, and P.~Frossard, ``Deepfool: {A} simple and
  accurate method to fool deep neural networks,'' in \emph{IEEE Conference on
  Computer Vision and Pattern Recognition (CVPR)}, 2016, pp. 2574--2582.

\bibitem{madry2018pgd}
A.~Madry, A.~Makelov, L.~Schmidt, D.~Tsipras, and A.~Vladu, ``Towards deep
  learning models resistant to adversarial attacks,'' in \emph{International
  Conference on Learning Representations (ICLR)}, 2018, pp. 1--23.

\bibitem{carlini2017cwl2}
N.~Carlini and D.~Wagner, ``Towards evaluating the robustness of neural
  networks,'' in \emph{IEEE Symposium on Security and Privacy (SP)}, 2017, pp.
  39--57.

\bibitem{croce2020aa}
F.~Croce and M.~Hein, ``Reliable evaluation of adversarial robustness with an
  ensemble of diverse parameter-free attacks,'' in \emph{International
  Conference on Machine Learning (ICML)}, 2020, pp. 2206--2216.

\bibitem{naseer2020ssp}
M.~Naseer, S.~H. Khan, M.~Hayat, F.~S. Khan, and F.~Porikli, ``A
  self-supervised approach for adversarial robustness,'' in \emph{IEEE
  Conference on Computer Vision and Pattern Recognition (CVPR)}, 2020, pp.
  259--268.

\bibitem{wong2020fgsmrs}
E.~Wong, L.~Rice, and J.~Z. Kolter, ``Fast is better than free: Revisiting
  adversarial training,'' in \emph{International Conference on Learning
  Representations (ICLR)}, 2020, pp. 1--17.

\bibitem{xu2022tnnls}
M.~Xu, T.~Zhang, Z.~Li, and D.~Zhang, ``Infoat: Improving adversarial training
  using the information bottleneck principle,'' \emph{IEEE Transactions on
  Neural Networks and Learning Systems (TNNLS)}, 2022.

\bibitem{yang2021ddrrn}
S.~Yang, T.~Guo, Y.~Wang, and C.~Xu, ``Adversarial robustness through
  disentangled representations,'' in \emph{Conference on Artificial
  Intelligence (AAAI)}, 2021, pp. 3145--3153.

\bibitem{zhai2022tnnls}
P.~Zhai and S.~Zhang, ``Adversarial information bottleneck,'' \emph{IEEE
  Transactions on Neural Networks and Learning Systems (TNNLS)}, 2022.

\bibitem{zhang2019trades}
H.~Zhang, Y.~Yu, J.~Jiao, E.~P. Xing, L.~E. Ghaoui, and M.~I. Jordan,
  ``Theoretically principled trade-off between robustness and accuracy,'' in
  \emph{International Conference on Machine Learning (ICML)}, vol.~97, 2019,
  pp. 7472--7482.

\bibitem{cao2019}
Y.~Cao and Q.~Gu, ``Generalization bounds of stochastic gradient descent for
  wide and deep neural networks,'' \emph{Advances in Neural Information
  Processing Systems (NeurIPS)}, vol.~32, pp. 10\,835--10\,845, 2019.

\bibitem{gowal2020wa}
S.~Gowal, C.~Qin, J.~Uesato, T.~A. Mann, and P.~Kohli, ``Uncovering the limits
  of adversarial training against norm-bounded adversarial examples,''
  \emph{CoRR}, vol. abs/2010.03593, 2020.

\bibitem{wu2021wider}
B.~Wu, J.~Chen, D.~Cai, X.~He, and Q.~Gu, ``Do wider neural networks really
  help adversarial robustness?'' in \emph{Advances in Neural Information
  Processing Systems (NeurIPS)}, pp. 7054--7067.

\bibitem{huang2021wideresnetR}
H.~Huang, Y.~Wang, S.~M. Erfani, Q.~Gu, J.~Bailey, and X.~Ma, ``Exploring
  architectural ingredients of adversarially robust deep neural networks,'' in
  \emph{Advances in Neural Information Processing Systems (NeurIPS)}, 2021, pp.
  5545--5559.

\bibitem{wu2020awp}
D.~Wu, S.~Xia, and Y.~Wang, ``Adversarial weight perturbation helps robust
  generalization,'' in \emph{Advances in Neural Information Processing Systems
  (NeurIPS)}, 2020.

\bibitem{addepalli2020dajat}
S.~Addepalli, S.~Jain \emph{et~al.}, ``Efficient and effective augmentation
  strategy for adversarial training,'' in \emph{Advances in Neural Information
  Processing Systems (NeurIPS)}, 2022, pp. 1488--1501.

\bibitem{krizhevsky2012alex}
A.~Krizhevsky, I.~Sutskever, and G.~E. Hinton, ``Imagenet classification with
  deep convolutional neural networks,'' in \emph{Advances in Neural Information
  Processing Systems (NeurIPS)}, 2012, pp. 1106--1114.

\bibitem{simonyan2015vgg}
K.~Simonyan and A.~Zisserman, ``Very deep convolutional networks for
  large-scale image recognition,'' in \emph{International Conference on
  Learning Representations (ICLR)}, 2015, pp. 1--14.

\bibitem{he2016resnet}
K.~He, X.~Zhang, S.~Ren, and J.~Sun, ``Deep residual learning for image
  recognition,'' in \emph{IEEE Conference on Computer Vision and Pattern
  Recognition (CVPR)}, 2016, pp. 770--778.

\bibitem{he2016preactresnet}
------, ``Identity mappings in deep residual networks,'' in \emph{European
  Conference on Computer Vision (ECCV)}, vol. 9908, pp. 630--645.

\bibitem{dosovitskiy2021vits}
A.~Dosovitskiy, L.~Beyer, A.~Kolesnikov, D.~Weissenborn, X.~Zhai,
  T.~Unterthiner, M.~Dehghani, M.~Minderer, G.~Heigold, S.~Gelly, J.~Uszkoreit,
  and N.~Houlsby, ``An image is worth 16x16 words: Transformers for image
  recognition at scale,'' in \emph{International Conference on Learning
  Representations (ICLR)}, 2021, pp. 1--21.

\bibitem{wang2020mart}
Y.~Wang, D.~Zou, J.~Yi, J.~Bailey, X.~Ma, and Q.~Gu, ``Improving adversarial
  robustness requires revisiting misclassified examples,'' in
  \emph{International Conference on Learning Representations (ICLR)}, 2020, pp.
  1--14.

\bibitem{zhang2020fat}
J.~Zhang, X.~Xu, B.~Han, G.~Niu, L.~Cui, M.~Sugiyama, and M.~S. Kankanhalli,
  ``Attacks which do not kill training make adversarial learning stronger,'' in
  \emph{International Conference on Machine Learning (ICML)}, vol. 119, 2020,
  pp. 11\,278--11\,287.

\bibitem{pang2020bag}
T.~Pang, X.~Yang, Y.~Dong, H.~Su, and J.~Zhu, ``Bag of tricks for adversarial
  training,'' in \emph{International Conference on Learning Representations
  (ICLR)}, 2021, pp. 1--21.

\bibitem{kanai2023relation}
S.~Kanai, M.~Yamada, H.~Takahashi, Y.~Yamanaka, and Y.~Ida, ``Relationship
  between nonsmoothness in adversarial training, constraints of attacks, and
  flatness in the input space,'' \emph{IEEE Transactions on Neural Networks and
  Learning Systems (TNNLS)}, 2023.

\bibitem{cohen2019defense}
J.~Cohen, E.~Rosenfeld, and Z.~Kolter, ``Certified adversarial robustness via
  randomized smoothing,'' in \emph{International Conference on Machine Learning
  (ICML)}, 2019, pp. 1310--1320.

\bibitem{salman2019smooth}
H.~Salman, J.~Li, I.~P. Razenshteyn, P.~Zhang, H.~Zhang, S.~Bubeck, and
  G.~Yang, ``Provably robust deep learning via adversarially trained smoothed
  classifiers,'' in \emph{Advances in Neural Information Processing Systems
  (NeurIPS)}, 2019, pp. 11\,289--11\,300.

\bibitem{liu2021provably}
C.~Liu, M.~Salzmann, and S.~S{\"{u}}sstrunk, ``Training provably robust models
  by polyhedral envelope regularization,'' \emph{IEEE Transactions on Neural
  Networks and Learning Systems (TNNLS)}, vol.~34, no.~6, pp. 3146--3160, 2023.

\bibitem{rebuffi2021fix}
S.~Rebuffi, S.~Gowal, D.~A. Calian, F.~Stimberg, O.~Wiles, and T.~A. Mann,
  ``Fixing data augmentation to improve adversarial robustness,'' \emph{CoRR},
  vol. abs/2103.01946, 2021.

\bibitem{lee2023booster}
H.~J. Lee, Y.~Yu, and Y.~M. Ro, ``Advancing adversarial training by injecting
  booster signal,'' \emph{IEEE Transactions on Neural Networks and Learning
  Systems (TNNLS)}, 2023.

\bibitem{wang2023diffusion}
Z.~Wang, T.~Pang, C.~Du, M.~Lin, W.~Liu, and S.~Yan, ``Better diffusion models
  further improve adversarial training,'' in \emph{International Conference on
  Machine Learning (ICML)}, vol. 202, 2023, pp. 36\,246--36\,263.

\bibitem{dai2021pssilu}
S.~Dai, S.~Mahloujifar, and P.~Mittal, ``Parameterizing activation functions
  for adversarial robustness,'' in \emph{IEEE Security and Privacy Workshops
  (SPW)}, 2022, pp. 80--87.

\bibitem{shao2021robustofvit}
R.~Shao, Z.~Shi, J.~Yi, P.~Chen, and C.~Hsieh, ``On the adversarial robustness
  of vision transformers,'' \emph{Transactions on Machine Learning Research
  (TMLR)}, vol. 2022, 2022.

\bibitem{zecchin2023emsemble}
M.~Zecchin, S.~Park, O.~Simeone, M.~Kountouris, and D.~Gesbert, ``Robust
  pac$^\text{m}$: Training ensemble models under misspecification and
  outliers,'' \emph{IEEE Transactions on Neural Networks and Learning Systems
  (TNNLS)}, 2023.

\bibitem{peng2023rwide}
S.~Peng, W.~Xu, C.~Cornelius, M.~Hull, K.~Li, R.~Duggal, M.~Phute, J.~Martin,
  and D.~H. Chau, ``Robust principles: Architectural design principles for
  adversarially robust cnns,'' in \emph{British Machine Vision Conference
  (BMCV)}, 2023.

\bibitem{pang2020mmc}
T.~Pang, K.~Xu, Y.~Dong, C.~Du, N.~Chen, and J.~Zhu, ``Rethinking softmax
  cross-entropy loss for adversarial robustness,'' in \emph{International
  Conference on Learning Representations (ICLR)}, 2020, pp. 1--19.

\bibitem{rice2020overfitting}
L.~Rice, E.~Wong, and J.~Z. Kolter, ``Overfitting in adversarially robust deep
  learning,'' in \emph{International Conference on Machine Learning (ICML)},
  vol. 119, 2020, pp. 8093--8104.

\bibitem{zagoruyko2016wideresnet}
S.~Zagoruyko and N.~Komodakis, ``Wide residual networks,'' in \emph{British
  Machine Vision Conference (BMCV)}, 2016.

\bibitem{boucheron2013ieq}
S.~Boucheron, G.~Lugosi, and P.~Massart, \emph{Concentration Inequalities - {A}
  Nonasymptotic Theory of Independence}.\hskip 1em plus 0.5em minus 0.4em\relax
  Oxford University Press, 2013.

\bibitem{rudin1978real}
W.~Rudin, \emph{Real and complex analysis}.\hskip 1em plus 0.5em minus
  0.4em\relax McGraw-Hill, 1978.

\bibitem{zeiler2010deconv}
M.~D. Zeiler, D.~Krishnan, G.~W. Taylor, and R.~Fergus, ``Deconvolutional
  networks,'' in \emph{IEEE Conference on Computer Vision and Pattern (CVPR)},
  2010, pp. 2528--2535.

\bibitem{goodfellow2016deeplearning}
I.~Goodfellow, Y.~Bengio, and A.~Courville, \emph{Deep learning}.\hskip 1em
  plus 0.5em minus 0.4em\relax MIT press, 2016.

\bibitem{athalye2018obfuscation}
A.~Athalye, N.~Carlini, and D.~Wagner, ``Obfuscated gradients give a false
  sense of security: Circumventing defenses to adversarial examples,'' in
  \emph{International Conference on Machine Learning (ICML)}, 2018, pp.
  274--283.

\bibitem{krizhevsky2009cifar}
A.~Krizhevsky, G.~Hinton \emph{et~al.}, ``Learning multiple layers of features
  from tiny images,'' \emph{Technical Report}, 2009.

\bibitem{hendrycks2019cifarc}
D.~Hendrycks and T.~G. Dietterich, ``Benchmarking neural network robustness to
  common corruptions and perturbations,'' in \emph{International Conference on
  Learning Representations (ICLR)}, 2019, pp. 1--16.

\bibitem{kannan2018alp}
H.~Kannan, A.~Kurakin, and I.~J. Goodfellow, ``Adversarial logit pairing,''
  \emph{CoRR}, vol. abs/1803.06373, 2018.

\bibitem{foolbox2017}
J.~Rauber, W.~Brendel, and M.~Bethge, ``Foolbox v0.8.0: {A} python toolbox to
  benchmark the robustness of machine learning models,'' \emph{CoRR}, vol.
  abs/1707.04131, 2017.

\bibitem{glorot2010xavier}
X.~Glorot and Y.~Bengio, ``Understanding the difficulty of training deep
  feedforward neural networks,'' in \emph{International Conference on
  Artificial Intelligence and Statistics (AISTATS)}, 2010, pp. 249--256.

\bibitem{tsne2009van}
L.~Van~der Maaten and G.~Hinton, ``Visualizing data using t-sne.''
  \emph{Journal of Machine Learning Research (JMLR)}, vol.~9, no.~11, 2008.

\end{thebibliography}

\clearpage

\appendix
\renewcommand{\appendixname}{Appendix}
\setcounter{prop}{0}
\setcounter{coro}{0}

\section{Experiment Details}
\label{appendix-experimental-details}

\subsection{Modified architectures of AlexNet and VGG}
\label{section-mod-arch}

\begin{figure}
	\centering
	\subfloat[AlexNet]{\includegraphics[width=0.45\textwidth]{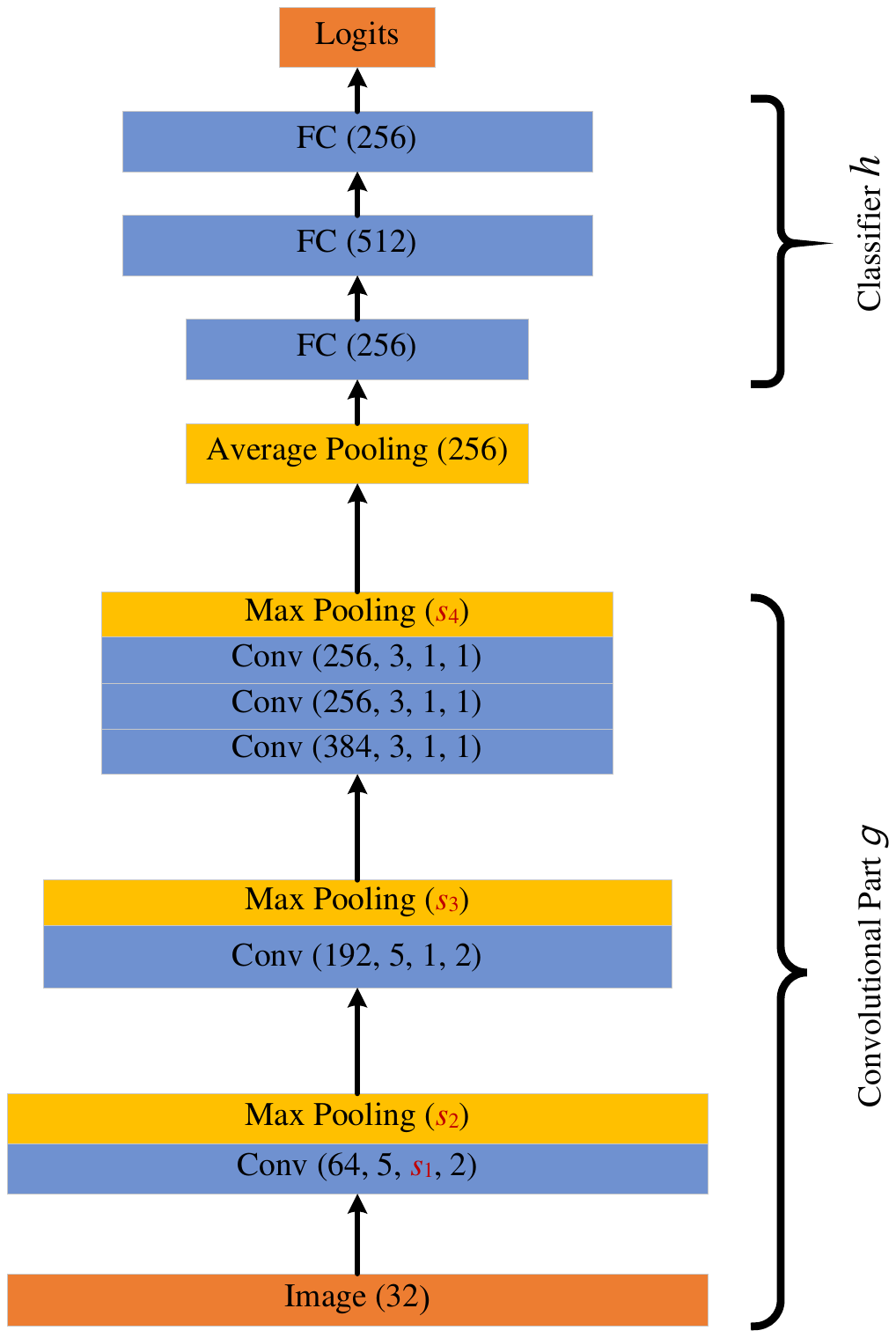}} \\
	\subfloat[VGG16]{\includegraphics[width=0.45\textwidth]{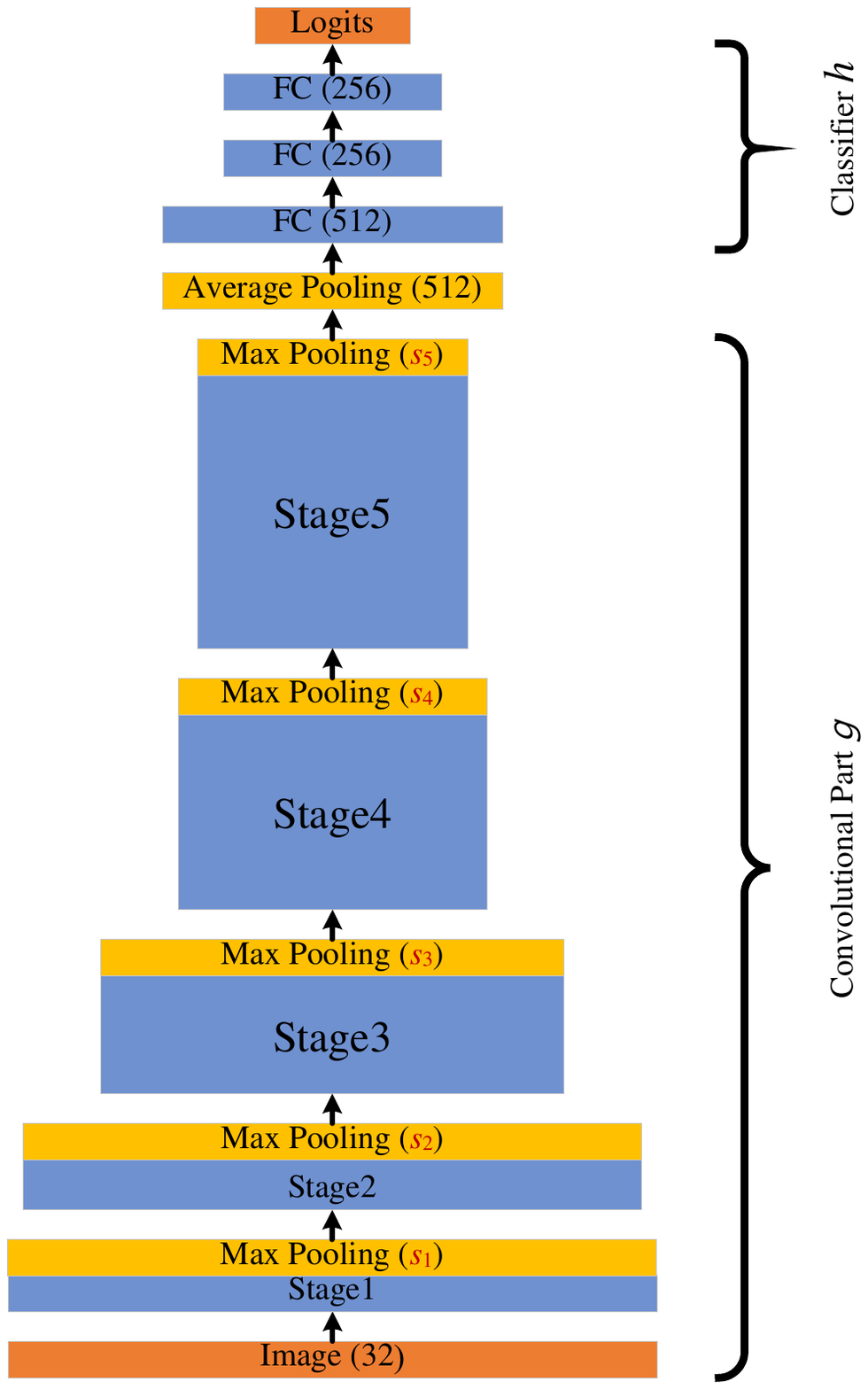}}
	\caption{
		The frameworks of AlexNet and VGG16.
		(a) AlexNet with the \textcolor{red}{$s_1$-$s_2$-$s_3$-$s_4$} sliding stride configuration.
		The quadruple notation therein represents channels, kernel size, stride and padding, respectively;
		(b) VGG16 with the \textcolor{red}{$s_1$-$s_2$-$s_3$-$s_4$-$s_5$} sliding stride configuration.
	}
	\label{appendix-fig-architectural-details}
\end{figure}

We provide the architectural details of AlexNet \cite{krizhevsky2012alex} and VGG \cite{simonyan2015vgg} in Figure \ref{appendix-fig-architectural-details},
where the adpative average pooling is performed after the convolutional part.

\subsection{Ablation study on filter size}

\begin{table}[htb]
	\caption{
		The impact of convolutional filter sizes on the effectiveness of the proposed modifications.
	}
	\label{table-filter-size}
	\centering
	\scalebox{.9}{
		\begin{tabular}{c|c|ccc}
			\toprule
			Filter Size & Configuration & Natural	&  PGD-20 \\
			\midrule
			\multirow{2}{*}{2} & 1-2-2-2 & 82.17 & 53.59 \\
			   &\underline{1-1-2-2}& 83.79 & 54.83 \\
			\midrule
			\multirow{2}{*}{3} & 1-2-2-2 & 82.56 & 53.50 \\
			   &\underline{1-1-2-2}& \textbf{84.13} & 55.37 \\
			\midrule
			\multirow{2}{*}{4} & 1-2-2-2 & 82.10 & 52.97 \\
			   &\underline{1-1-2-2}& 83.92 & \textbf{55.44} \\
			\bottomrule
	\end{tabular}}
\end{table}

Table \ref{table-filter-size} reports an ablation study on the convolutional filter size.
It can be seen that the proposed modifications are efficient across different filter sizes.

\subsection{High-resolution dataset}
\label{section-high-resolution}

\begin{table}[htb]
	\caption{
		Natural Accuracy (\%) and Adversarial Robustness (\%) under PGD-20 attack on ImageNette.		
	}
	\label{table-imagenette}
	\centering
	\scalebox{.9}{
		\begin{tabular}{c|c|c|ccc}
			\toprule
			& Image Size & Configuration  & Natural	&  PGD-20 \\
			\midrule
			\multirow{3}{*}{ResNet18} & \multirow{3}{*}{128} & 2-2-2-2 & 87.47 & 47.13  \\
			   & &1-2-2-2& 87.44 & 46.85 \\
			   & &2-1-2-2& \textbf{87.83} & \textbf{47.80} \\
			\midrule
			\multirow{3}{*}{ResNet18} & \multirow{3}{*}{64} & 2-2-2-2 & 84.59 & 40.00  \\
			   & &2-1-2-2& 85.66 & \textbf{42.93} \\
			   & &1-2-2-2& \textbf{85.86} & 42.27 \\
			\midrule
			\multirow{2}{*}{VGG19} & \multirow{2}{*}{128} & 2-2-2-2-2 &\textbf{79.62} & 39.19  \\
			   & &1-2-2-2-2& 79.57 & \textbf{41.66} \\
			\bottomrule
	\end{tabular}}
\end{table}

\begin{table}[htb]
	\caption{
		The impact of pooling type on the high-resolution dataset of ImageNette with the $128 \times 128$ image size.
		`-' indicates the case that cannot be trained successfully by TRADES.
	}
	\bigskip
	\label{table-imagenette-pooling}
	\centering
	\scalebox{.9}{
		\begin{tabular}{c|c|c|ccc}
			\toprule
			& Pooling Type & Configuration  & Natural	&  PGD-20 \\
			\midrule
			\multirow{4}{*}{ResNet18} & Average & 2-2-2-2 & \textbf{87.47} & \textbf{47.13}  \\
				& Max & 2-2-2-2 & 84.05	& 41.15 \\
			   & Average &1-2-2-2& 87.44 & 46.85 \\
				& Max & 1-2-2-2 & -	& - \\
			\midrule
			\multirow{4}{*}{VGG19} & Average & 2-2-2-2-2 & \textbf{79.62} & 39.19  \\
				& Max & 2-2-2-2-2 & 71.46 &  32.69    \\
			   & Average &1-2-2-2-2& 79.57 & \textbf{41.66} \\
				& Max & 1-2-2-2-2 & - &  -   \\
			\bottomrule
	\end{tabular}}
\end{table}

ImageNette-10\footnote{\url{https://github.com/fastai/imagenette}} dataset, comprising of 10 classes of $128 \times 128$ images, is used to investigate whether the proposed modifications could be extended to high-resolution datasets.

From the experimental results in Tables \ref{table-imagenette} and \ref{table-imagenette-pooling}, we have the following observations:
\begin{itemize}
\item
Although the 1-2-2-2 and 2-1-2-2 configurations correspond to the same feature map size, the latter is more effective on high-resolution data.
This is due to earlier downsampling.
The configuration of 2-1-2-2 can thus have a better receptive field than 1-2-2-2,
which is critical for the learning of high-resolution data.
In fact, the improvements of the 1-2-2-2 configuration can be observed again when the image resolution is reduced to $64 \times 64 $.
\item
Compared to ResNet18, VGG19 which has more downsampling operations followed by a chain of fully connected layers,
thus can alleviate the above problem and benefits from the modification.
\item
While the experimental results are not as impressive as for low-resolution dataset,
we can see from Table \ref{table-imagenette-pooling} that the conclusion that average pooling has better robustness than maximum pooling  still holds.
\end{itemize}
Overall, enlarging feature maps for high-resolution datasets should be performed with more care and with full consideration of receptive field degradation.
Deciding an optimal stride configuration in concert with a sufficient receptive field is an interesting problem that deserves further research.

\end{document}